\documentclass[12pt]{alt2020}

\usepackage{times}
\usepackage{mathtools}
\usepackage{amssymb,amsfonts,amsmath}
\usepackage{algorithm}
\usepackage{algorithmic}
\newtheorem{prop}{Proposition}
\DeclarePairedDelimiter{\ceil}{\lceil}{\rceil}

\title[Planning in Hierarchical Reinforcement Learning]{Planning in Hierarchical Reinforcement Learning:\\ Guarantees for Using Local Policies}

\altauthor{\Name{Tom Zahavy} \Email{tomzahavy@google.com}\\
 \Name{Avinatan Hasidim} \Email{avinatan@google.com }\\
 \Name{Haim Kaplan} \Email{haimk@google.com }\\
 \Name{Yishay Mansour} \Email{mansour@google.com }\\
 \addr Google Research, Tel Aviv}
\begin{document}

% \author{\centering \name Tom Zahavy$^{1,2}$, Avinatan Hasidim$^{1,3}$, Haim Kaplan$^{1,4}$, Yishay Mansour$^{1,4}$\\
% \addr Google$^1$, Technion Israel Institute of Technology$^2$, Bar Ilan Univeristy$^3$, Tel Aviv University$^4$\\
% }
\maketitle

\begin{abstract}
We consider a setting of hierarchical reinforcement learning, in which the reward is a sum of components. For each component, we are given a policy that maximizes it, and our goal is to assemble a policy from the individual policies that maximize the sum of the components. We provide theoretical guarantees for assembling such policies in deterministic MDPs with collectible rewards. Our approach builds on formulating this problem as a traveling salesman problem with a discounted reward. We focus on local solutions, i.e., policies that only use information from the current state; thus, they are easy to implement and do not require substantial computational resources. We propose three local stochastic policies and prove that they guarantee better performance than any deterministic local policy in the worst case; experimental results suggest that they also perform better on average. 
\end{abstract}

\section{Introduction}
One of the unique characteristics of human problem solving is the ability to represent the world in different granularities. When we plan a trip, we first choose the destinations we want to visit and only then decide what to do at each destination. Hierarchical reasoning enables us to break a complex task into simpler ones that are computationally tractable to reason about. Nevertheless, the most successful Reinforcement Learning (RL) algorithms perform planning in a single abstraction level. 

RL provides a general framework for optimizing decisions in dynamic environments. However, scaling it to real-world problems suffers from the curse of dimensionality, that is, coping with large state spaces, action spaces, and long horizons. The most common approach to deal with large state spaces is to approximate the value function or the policy, making it possible to generalize across different states \citep{tesauro1995temporal,mnih2015human}. For long horizons, combining Monte Carlo simulation with value and policy networks was shown to search among game outcomes efficiently, leading to a super-human performance in playing Go, Chess, and Poker \citep{silver2016mastering,moravvcik2017deepstack}. Another long-standing approach for dealing with long horizons is to introduce hierarchy into the problem (see \citep{barto2003recent} for a survey). In particular,  \citet{sutton1999between} extended the RL formulation to include options -- local policies for taking actions over a period of time. An option $o$ is formally defined as a three-tuple $\{ I,\pi,\beta \}$, where $I$ is a set of option initiation states, $\pi$ is the option policy, and $\beta$ is a set of option termination states. The options framework presents a two-level hierarchy, where options to achieve sub-goals are learned separately, and policy over options selects among options to accomplish the global goal. Hierarchical solutions based on this formulation simplify the problem and demonstrate superior performance in challenging environments \citep{tessler2017deep,vezhnevets2017feudal,bacon2017option}.

% reward decomposition
In this work, we focus on a specific type of hierarchy - reward function decomposition - that dates back to the works of \citep{humphrys1996action,karlsson1997learning} and was recently combined with deep learning \citep{van2017hybrid}. In this formulation, the goal of each option $i$ is to maximize a local reward function $R_i$, while the final goal is to maximize the sum of rewards $R_M=\sum R_i$. Each option is trained separately and provides a policy and its value function. The goal is to learn a policy over options that uses the value functions of the options to select among them. That way, each option is responsible for solving a simple task, and the options' optimal policies are learned in parallel on different machines. 

Formally, a set of options defined over a Markov decision process (MDP) constitutes a semi-MDP (SMDP), and the theory of SMDPs provides the foundation for the theory of options. In particular, the policy over options can be learned using an SMDP algorithm \citep{sutton1999between}. A different approach is to use predefined rules to select among options. Such rules allow us to derive policies for the MDP $M$ (to maximize the sum of rewards) by learning options (without learning directly in  $M$), such that learning is fully decentralized. For example, choosing the option with the largest value function \citep{humphrys1996action,barreto2016successor}, or choosing the action that maximizes the sum of
the values given to it by the options \citep{karlsson1997learning}. The goal of this work is to provide theoretical guarantees for using such rules. 

We focus on a specific case where the dynamics are deterministic, and the individual rewards correspond to collectible items, which is common in navigation benchmarks \citep{tessler2017deep,beattie2016deepmind}. We denote such an MDP by $M$. The challenge with collectible rewards is that the state changes each time we collect a reward (the subset of available rewards is part of the state). Notice that if this information is not included in the state, then it is not a Markov state. Since all the combinations of remaining items have to be considered, the state space grows exponentially with the number of rewards. We, therefore, focus on solving an SMDP $M_s$ that is a composition of $M$ with a set of (optimal) options for collecting each one of the rewards. Each option can be initiated in any state, its policy is to collect a single reward, and it terminates once it collected it. Finding the optimal policy in $M_s$ remains a hard problem; the size of the state space reduces but remains exponential.  

We address this issue in two steps. First, we show that finding the optimal policy in $M_s$ (the optimal policy over options) is equivalent to solving a Reward Discounted Traveling Salesman Problem (RD-TSP). Then, we derive and analyze approximate solutions to the RD-TSP. Similar to the classical TSP, the goal in the RD-TSP is to find an order to collect all the rewards; but instead of finding the shortest tour, we maximize the discounted cumulative sum of the rewards (Definition \ref{def:3}). Not surprisingly, computing an optimal solution to the RD-TSP is NP-hard \citep{blum2007approximation}. 
A brute force approach for solving the RD-TSP requires evaluating all the $n!$ possible tours connecting the $n$ rewards; an adaptation of the Bellman–Held–Karp algorithm\footnote{Dynamic programming solution to the TSP \citep{bellman1962dynamic,held1962dynamic}.} to the RD-TSP (Algorithm 4 in the supplementary material) is identical to tabular Q-learning on the SMDP $M_s$, and requires exponential time. This makes the task of computing the optimal policy for our SMDP infeasible.\footnote{The Hardness results for RD-TSP do not rule out efficient solutions for special MDPs. In the supplementary, we provide exact polynomial-time solutions for the case in which the MDP is a line and when it is a star.}  \cite{blum2007approximation} proposed a polynomial-time algorithm for RD-TSP that computes a policy which collects at least $0.15$ fraction of the optimal discounted return, which was later improved to $0.19$  \citep{farbstein2016discounted}. These algorithms need to know the entire SMDP to compute their approximately optimal policies.

In contrast, we focus on deriving and analyzing approximate solutions that use only \textbf{local} information, i.e., they only observe the value of each option from the current state. Such \textbf{local policies} are straightforward to implement and are computationally efficient. The reinforcement learning community is already using simple local approximation algorithms for hierarchical RL \cite{borsa2018universal,hansen2019fast,barreto2016successor,barreto2018transfer,barreto2019option,van2017hybrid}. We hope that our research provides theoretical support for comparing local heuristics, and in addition introduces new reasonable local heuristics. Specifically, we prove worst-case guarantees on the reward collected by these algorithms relative to the optimal solution (given optimal options). We also prove bounds on the maximum reward that such local policies can collect. In our experiments, we compare the performance of these local policies in the planning setup (where all of our assumptions hold), and also during learning (with suboptimal options), and with stochastic dynamics.
%
%In particular, our main contributions are as follows.

{\bf Our results:} We establish impossibility results for local policies, showing that no deterministic local policy can guarantee a reward larger than $ 24 \text{OPT} / n $ for any  MDP, and no stochastic policy can guarantee a reward larger than $8 \text{OPT} / \sqrt{n}$ (where OPT denotes the value of the optimal solution). These impossibility results imply that the Nearest Neighbor (NN) algorithm that iteratively collects the closest reward (and thereby a total of at least $\text{OPT}/n$ reward) is optimal up to a constant factor amongst all deterministic local policies.

On the positive side, we propose three simple stochastic policies that outperform NN. The best of them combines NN with a Random Depth First Search (RDFS) and guarantees performance of at least $\Omega \left( \text{OPT} / \sqrt{n} \right)$ when OPT achieves $\Omega \left(n \right)$,  and at least $\Omega (\text{OPT} / n^{2/3} )$ in the general case.
Combining NN with jumping to a random reward and sorting the rewards by their distance from it, has a slightly worse guarantee.
A simple modification of  NN to first jump to a random reward and continues NN from there already improves the guarantee to $\Omega(\text{OPT} \log(n) /n)$.

\section{Problem formulation}
\label{sec:problem}
We consider the standard RL formulation \citep{sutton2018reinforcement}, and focus on MDPs with deterministic dynamics and reward that is decomposed from a sum of collectible rewards (Definition \ref{def:2}). Models that satisfy these properties appear in numerous domains including many maze navigation problems, the Arcade Learning Environment, and games like Chess and Go \citep{bellemare2013arcade,barreto2016successor,tessler2017deep,van2017hybrid}.
\begin{definition}[Collectible Reward Decomposition MDP]
\label{def:2}
An MDP $M$ is called a Collectible Reward Decomposition MDP if it satisfies the following properties: \textbf{(1) Reward Decomposition,} the reward in $M$ represents the sum of the local rewards: $R_M = \sum _{i=1} ^ n R_i$;
\textbf{(2) Collectible Rewards,} each reward signal $\{R_i\}_{i=1}^n$ represents a single collectible reward, i.e., $R_i(s) = 1$ iff $s = s_i$ for some particular state $s_i$ and $R_i(s) = 0$ otherwise. In addition, each reward can only be collected once.
\textbf{(3) Its dynamic is deterministic.} %$P$ is a deterministic transition matrix, i.e., for each action $a,$ each row of $P^a$ has exactly one value that equals $1$, and all other values equal zero.
\end{definition}
%Property 1 in Definition \ref{def:2} requires $R_M$ to be a decomposition of the previous rewards, and Property 2 requires each local reward to be a collectible reward. 

For the state to be Markov, we extend the state to include information regarding all the rewards that were collected so far. Since all the combinations of remaining items have to be considered, the state space grows exponentially with the number of rewards. To address this issue, we now propose two reductions from the MDP $M$ into an SMDP $M_s$ and then into an RD-TSP. 

We consider the setup where we are given a set of optimal local policies (options) one per reward. In practice, such options can be learned (in a distributional manner), but for now, we assume that they are given.  Option $j$ tells us how to get to reward $j$ via the shortest path. To take decisions, when we are at state $i$, then we  observe the value $V_j(i)$ of  $o_j$ at state $i$ for each $j$. We have  that  $V_j(i) = \gamma ^ {d_{i,j}}$ where $d_{i,j}$ is the length of the shortest path from $i$ to $j$. Notice that at any state, an optimal policy in $M$  always follows the shortest path to one of the rewards.\footnote{To see this, assume there exists a policy $\mu$ that does not follow the shortest path from some state $k$ to the next reward-state $k'$. Then, we can improve $\mu$ by taking the shortest path from $k$ to $k'$, contradicting the optimality of $\mu$. This implies that an optimal policy on $M$ is a composition of the local options $\{o _i \}_{i=1}^n$.} Given that the dynamics are deterministic, an optimal policy in $M$ makes decisions only at states which contain rewards. In other words, once the policy arrived at a reward-state $i$ and decided to go to a reward-state $j$, then it will follow the optimal policies $o_j$ until it reaches $j$.\footnote{This is not true if $P$ is stochastic. To see this, recall that the stochastic shortest path is only the shortest in expectation. Thus, stochasticity may lead to states in which it is better to change the target reward.}

For a \emph{collectible reward decomposition} MDP $M$ (Definition \ref{def:2}), an optimal policy can be derived in an SMDP \citep{sutton1999between} denoted by $M_s$. The state-space of $M_s$ contains only the initial state $s_0$ and the reward states $\{s_i \}_{i=1}^n$. The action space is replaced by the set of optimal options $\{o_i \}_{i=1}^n$ for collecting each one of the rewards. These options can start in any state of $M_s$ and terminate once they collect the reward. While this construction reduces the size of the state space to include only reward states, the size of the state space remains exponential (since all the combinations of remaining items have to be considered, the state space grows exponentially with the number of rewards). We will soon address this computational issue by deriving and analyzing approximate solutions that are computationally feasible. But before we do that, we quickly discuss the reduction from the MDP $M$ to the SMDP $M_s$.

%In addition, in this action space, the transition matrix $P$ is deterministic since $\forall s, a_i, \exists s'$ such that $P^{a_i}_{s,s'}=1$, and otherwise $P^{a_i}_{s,\cdot}=0$. Finally, the reward signal and the discount factor remain the same. {\bf HK i do not understand what is this last sentence connects to ? }

In general, optimal policies for SMDPs are not guaranteed to be optimal in the original MDP, i.e., they do not achieve the same reward as the optimal policy in $M$ that uses only primitive actions. One trivial exception is the case that the set of options in $M_s$ includes the set of primitive actions (and perhaps also options) \citep{sutton1999between}. Another example is landmark options \citep{mann2015approximate}, that plan to reach a specific state in a deterministic MDP. Given a set of landmark options (that does not include primitive actions), \citeauthor{mann2015approximate} showed that the optimal policy in $M_s$ is also optimal in $M$ (under a variation of Definition \ref{def:2}). In addition, they analyzed sub-optimal options and stochastic dynamics \citep{mann2015approximate}, which may help to extend our results in future work. 

Now that we understand that an optimal policy for $M$ can be found in $M_s$, we make another reduction, from $M_s$ to an RD-TSP (Definition \ref{def:3}).   

\begin{definition}[RD-TSP] 
\label{def:3}
Given an undirected graph with $n$ nodes and edges $e_{i,j}$ of lengths $d_{i,j}$, find a path (a set of indices $\{i_t \}_{t=1}^n$) that maximizes the discounted cumulative return:$$\enspace \enspace \{i_t^* \}_{t=1}^n=\underset{\{i_t \}_{t=1}^n \in \text{perm} \{ 1, .., n\} }{\text{arg max}}\sum_{j=0}^{n-1} \gamma ^ {\sum _{t=0}^{j} d_{i_t,i_{t+1}}}.$$
\end{definition}

Throughout the paper, we will compare the performance of different algorithms with that of the optimal solution. We will refer to the optimal solution and to its value as OPT (no confusion will arise). 
\begin{definition}[OPT]
\label{def:OPT}
Given an RD-TSP (Definition \ref{def:3}), the value of the optimal solution is $$\text{OPT} = \underset{\{i_t \}_{t=1}^n \in \text{perm} \{ 1, .., n\} }{\text{max}}\sum_{j=0}^{n-1} \gamma ^ {\sum _{t=0}^{j} d_{i_t,i_{t+1}}} .$$
\end{definition}

The following Proposition suggests that an optimal policy on $M_s$ can be derived by solving an RD-TSP (Definition \ref{def:3}). Here, a node corresponds to a state in $M_s$ (initial and reward states in $M$), and edge $e_{i,j}$ corresponds to choosing the optimal option $o_j $ in state $i$, and the length of an edge $d_{i,j}$ corresponds to time it takes to follow an option $j$ from state $i$ until it collects the reward in state $j$ ($\log_ \gamma V_j(i)$).

\begin{prop}[MDP to RD-TSP] 
\label{prop}
Given an MDP $M$ that satisfies Definition \ref{def:2} with $n$ rewards and a set of options $\left\{ o_i \right\}_{i=1}^n$ for collecting them, define a graph, $G$, with nodes corresponding to the initial state and the reward-states of $M$. Define the length $d_{i,j}$ of an edge $e_{i,j}$ in $G$ to be $\log_\gamma \big(V_j(i)\big)$, i.e., the logartihm of the value of following option $o_j$ from state $i$. Then, an optimal policy in $M$ can be derived by solving an RD-TSP in $G$. %\hk{I do not think this is true ? don't the length of the edge should be the length of the shortest path rather than $V_j(i)$ ?which is already $\gamma^{d_{i,j}}$}
\end{prop}

So far, we have shown that finding the optimal policy in $M$ is equivalent to finding it in $M_s$ or alternatively, by solving the equivalent RD-TSP. However, all of these reductions did not seem to be useful, as the size of the state space remained exponential. Nevertheless, the goal of these reductions will soon become clear; it allows us to derive approximate solutions to the RD-TSP and analyze their performance in $M$. 

We focus on approximate solutions in the class of \textbf{Local Policies}. A  local policy is  a mapping whose inputs are:
\begin{enumerate}
    \item The current state $x$.
    \item The history $h$ containing the previous steps taken by the policy; in particular, $h$ includes the rewards that we have already collected.
    \item The discounted return for each option from the current state, i.e., $\{V_i (x)\}_{i=1}^n$,
 and whose output is a distribution over the options. 
\end{enumerate}
Since local policies make decisions based on local information, they are straightforward to implement and are computationally efficient.

Notice that a local policy does not have full information on the MDP (but only on local distances).\footnote{Notice that the optimal global policy, (computed given the entire SMDP as input), has the same history dependence as local policies. However, while the global policy is deterministic, our results show that for local policies, stochastic policies are better than deterministic ones. It follows that the benefit of stochastic selection is due to the local information and not due to the dependence on history. This implies that locality is related to partial observability, for which it is known that stochastic policies can be better \citep{aumann1996absent}.} 
\begin{definition}
\label{def:4}
A local policy $\pi_{local}$ is a mapping: 
%from a reward state $x$, history $h$, and the option value functions evaluated at $x$, and the history $h$, to a distribution over options,
$\pi_{local} (x,h,\{V_i (x)\}_{i=1}^n) \rightarrow \Delta (\{o_i\}_{i=1}^n), $
where $\Delta(X)$ is the set of distributions over a finite set $X$.
\end{definition}

In the next sections, we propose and analyze specific local policies. We distinguish between policies that are deterministic and policies that are stochastic. For each class, we derive an impossibility result and propose policies that match it. 

\section{NN}
We start with an analysis of  the  natural Nearest Neighbor (NN) heuristics for the TSP. This algorithm chooses at each node the reward that is closest to the node (amongst the remaining rewards). In the context of our problem, NN is the policy that selects the option with the highest estimated value, exactly like GPI \citep{barreto2016successor}. We shall abuse the notation slightly and use the same name (e.g., NN) for the algorithm itself and its value; no confusion will arise. For TSP (without discount) in general graphs, we know that \citep{rosenkrantz1977analysis}:
\begin{center}
$\frac{1}{3} \log_2(n+1) +\frac{4}{9} \le \frac{\text{NN}}{\text{OPT}} \le \frac{1}{2} \ceil{\log_2(n)}+\frac{1}{2},$
\end{center}
where in the above equation, OPT refers to the value of the optimal solution in the TSP (and not that of the RD-TSP). However, for the RD-TSP, the NN algorithm only guarantees a value of $\text{OPT}/n,$ as the following Theorem states.
 %\tz{The intuition behind this lower bound is that deterministic policies can be tricked to choose a reward that is close to the origin but distant from the rest. }
\begin{theorem}[NN Performance]
\label{theo:1}
For any MDP satisfying Definition \ref{def:2} with $n$ rewards, and $\forall \gamma,$   $$\frac{\text{NN}}{\text{OPT}}\ge\frac{1}{n}.$$
\end{theorem} 

\begin{proof}
Denote by $i^*$ the nearest reward to the origin $s_0$, and by $d_{0,i^*}$ the distance from the origin to $i^*$.
The distance from $s_0$ to the first reward collected by OPT is at least  $d_{0,i^*}$.
 Thus, if $o_0=s_0,o_1,\ldots,o_{n-1}$ are the rewards ordered in the order by which OPT collects them we get that
\begin{align*}
\text{OPT} &= \sum\nolimits _{j=0} ^{n-1} \gamma ^ {\sum _{t=0}^{j} d_{o_t,o_{t+1}}} 
\leq \gamma ^{d_{0,i^*}} (1+ \sum\nolimits_{j=1}^{n-1} \gamma ^ {\sum _{t=1}^{j} d_{o_t,o_{t+1}}} ) \leq n \gamma ^ {d_{0,i^*}} 
\end{align*}
On the other hand, the NN heuristic chooses $i^*$ in the first round, thus, its cumulative reward is at least $\gamma ^{d_{0,i^*}}$ and we get that
$ \frac{\text{NN}}{\text{OPT}} \ge \frac{\gamma ^{d_{0,i^*}}}{n\gamma^{d_{0,i^*}} } = \frac{1}{n}.$
\end{proof}

\section{Impossibility Results for Deterministic Policies}

\begin{figure}[t]
\centering
\begin{subfigure}
    \centering
    \includegraphics[width=0.45\linewidth]{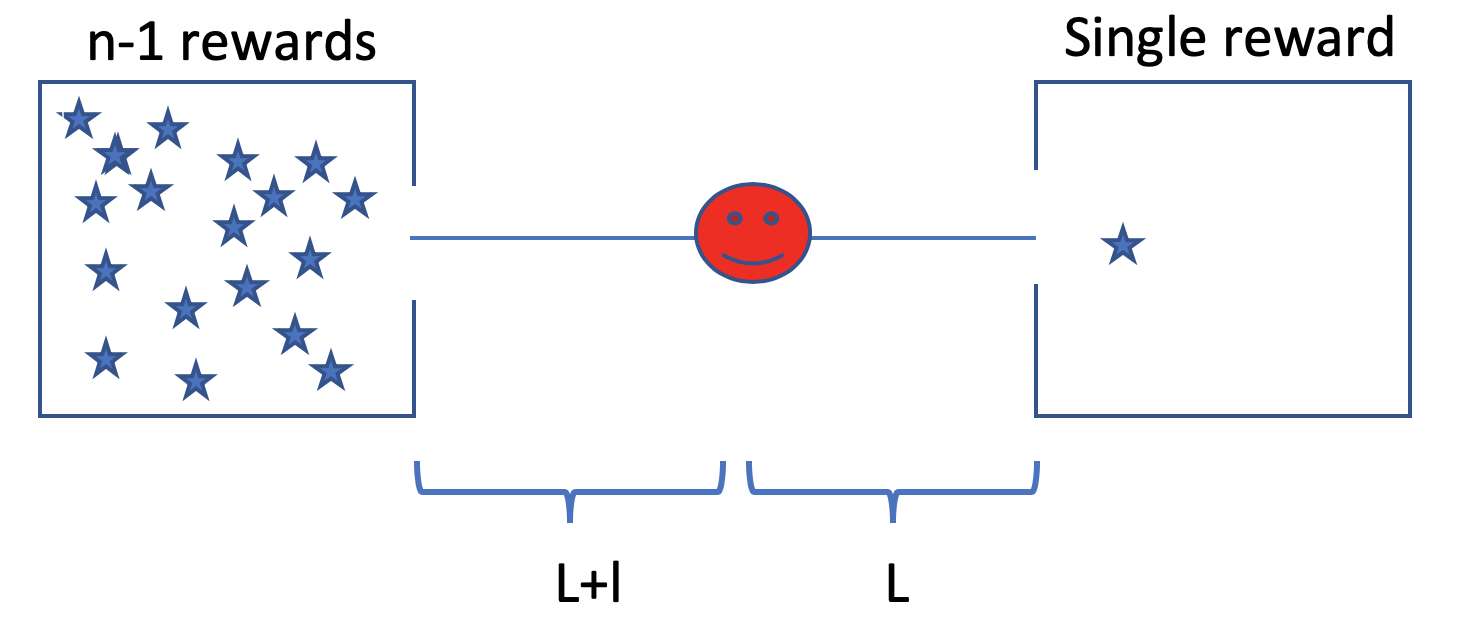}
\end{subfigure}%
\begin{subfigure}
    \centering
    \includegraphics[width=0.45\linewidth]{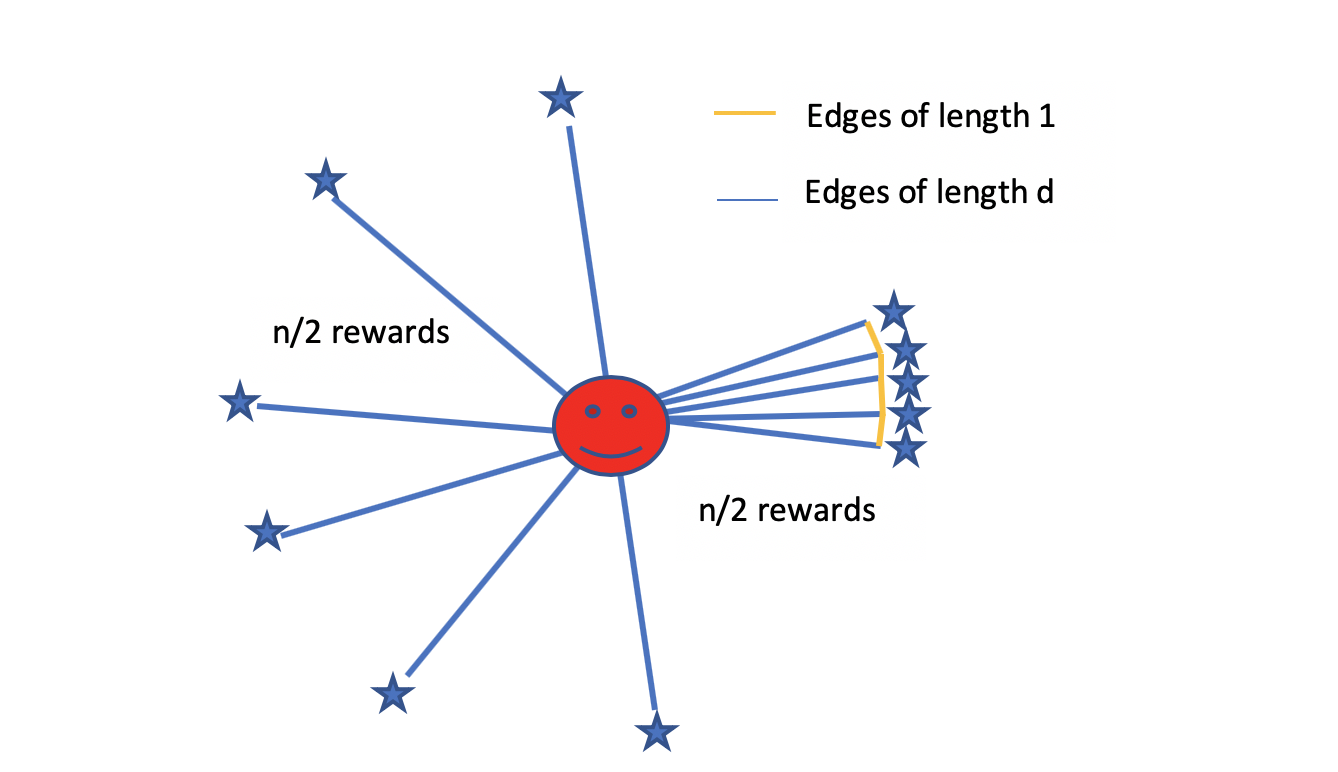}
\end{subfigure}%
\caption{\textbf{Impossibility results} for NN(left) and Deterministic local policies (right).  }
\label{fig:worst_case_mdp}
\end{figure}

The \textbf{NN} heuristic guarantees performance of at least $\text{OPT}/n.$ We now describe an MDP at which NN indeed cannot guarantee more than $\text{OPT}/n.$ Such an MDP is presented in Figure \ref{fig:worst_case_mdp} (left). The agent starts in the center, and there are $n-1$ rewards at a distance $L+l$ to its left and a single reward at a distance $L$ to its right. Inside the left "room," the rewards are connected with short edges (of length much smaller than  $L$). For large values of $L$ (larger than $l$ and the size of each room), only the rewards in the room that is visited first contribute to the collected value because of the discounting. Since the single reward is located slightly closer to the agent, NN collects it first. As a result, it only gets about $\text{OPT}/n$
reward. 

Next, we show an impossibility result for all \textbf{deterministic local policies}, indicating that no such policy can guarantee more than $\text{OPT}/n,$ which makes NN optimal over such policies. An example for the MDP that is used in the proof is presented in Figure \ref{fig:worst_case_mdp} (right).
\begin{theorem}[Impossibility for Deterministic Local Policies] 
\label{theo:2det}
For any deterministic local policy D-Local, $\exists$ MDP with $n$ rewards and a discount factor $\gamma = 1-1/n,$ s.t.:  $ \frac{\text{D-Local}}{\text{OPT}} \leq \frac{24}{n}. $
\end{theorem}

\begin{proof}
Consider a family of graphs, ${\cal G}$, each of which consists of a star with a central vertex and $n$ leaves. The starting vertex is the central vertex, and there is a reward at each leaf. The length of each edge is $d$ (s.t $\gamma^d = \frac{1}{2}$).

Each graph of the family ${\cal G}$ corresponds to a different subset of $n/2$ of the leaves, which we connect (pairwise) by edges of length $1$ (the other $n/2$ leaves are only connected to the central vertex). While at the central vertex, local policy cannot distinguish among the $n$ rewards (they all at the same distance from the origin), and therefore its choice is the same for all graphs in ${\cal G}$ (the next decision is also the same and so on, as long as it does not hit one of the $n/2$ special rewards).

Thus, for any given policy, there exists a graph in ${\cal G}$ such that the adjacent $n/2$ rewards are visited last. Since $\gamma = 1-\frac{1}{n}$ we have that 
$\frac{n}{4} \le \sum\nolimits_{i=0} ^{n/2-1} \gamma ^i = \frac{1-\gamma^{n/2}}{1-\gamma} \le \frac{n}{2},$ therefore
\begin{align*}
\frac{\text{D-Local}}{\text{OPT}} &= \frac{\gamma ^ {-d} \sum_{i=1} ^{n/2} \gamma ^ {(2i-1)d} +  \gamma ^ {nd+1} \sum_{i=0} ^{n/2-1} \gamma ^i}{ \sum_{i=0} ^{n/2-1} \gamma ^i + \gamma ^{2d+n/2-1} \sum_{i=1} ^{n/2} \gamma ^ {(2i-1)d}} %
\le \frac{\sum_{i=1} ^{n/2} \gamma ^ {(2i-1)d} +  0.5n \gamma ^ {nd+1} }{\gamma ^ {d} \sum_{i=0} ^{n/2-1} \gamma ^i } \\%
&= \frac{2 \sum_{i=1} ^{n/2} 0.25 ^ i +  \frac{0.5 ^ {n} n}{4}  }{0.5 \sum_{i=0} ^{n/2-1} \gamma ^i } \le 
\frac{6}{\sum_{i=0} ^{n/2-1} \gamma ^i } \le \frac{24}{n}.
\end{align*}
\end{proof}
\section{Random-NN}
In the previous sections we have seen that although that NN cannot guarantee more than $\frac{\text{OPT}}{n}$ rewards, it is actually optimal amongst deterministic local policies. Next, we propose a simple, easy to implement, stochastic adjustment to the NN algorithm with a better  bound which we call Random-NN (R-NN). The algorithm starts by collecting one of the rewards, say $s_1$, picked at random, and continues by executing NN (Algorithm \ref{alg:nn-random}). 

\begin{algorithm}
\caption{R-NN: NN with a First Random Pick}
\label{alg:nn-random}
\begin{algorithmic}
\STATE {\bfseries Input:} MDP M, with $n$ rewards, and $s_0$ the first reward
\STATE Flip a coin
\IF{outcome = heads} %$\#$Perform 
    \STATE \hskip0.5em Collect a random reward, denote it by $s_1$   
\ENDIF
\STATE Follow by executing NN 
\end{algorithmic}
\end{algorithm}

\begin{theorem}[R-NN Performance]
\label{theo:nn-rand}
For an MDP satisfying Definition \ref{def:2} with $n$ rewards, and $\forall \gamma$:  $ \frac{\text{R-NN}}{\text{OPT}} \ge \Omega \left( \frac{\log (n)}{n} \right).$
\end{theorem}

We now analyze the performance guarantees of the R-NN method. The analysis consists of two steps. In the first step, we assume that OPT achieved a value of $\Omega (n')$ by collecting $n'$ rewards and consider the case that $n' = \alpha n$. The second step considers the more general case and analyzes the performance of NN-Random for the worst value of $n'.$ 

\begin{proof} We start with the following lemma. 

\begin{lemma}
\label{lemma:theta}
For any path $p$ of length $x$, and $\forall \theta \in [0,x]$, there are less than $x/ \theta$ edges in $p$ that are larger than $\theta$. 
\end{lemma}

\begin{proof}
For contradiction, assume there are more than $x/ \theta$ edges longer than $\theta$. The length of $p$ is given by 
$\sum_i p_i = \sum _{p_i \le \theta } p_i + \sum _{p_i > \theta} p_i \ge 
\sum _{p_i \le \theta } p_i + \frac{\theta x}{\theta}$ thus a contradiction to the assumption that the path length is at most $x$. 
\end{proof}

\textbf{Step 1.}

Assume that OPT OPT achieves a value of $\Omega (n')$ by collecting $n'$ rewards for some fixed $0\le \alpha \le 1.$

Define $x = \mbox{log}_{1/\gamma} (2)$ and $\theta=x/\sqrt{n}$ (here we can replace the $\sqrt{n}$ by any fractional power of $n$, this will not affect the asymptotics of the result).

We denote a set of rewards that are connected by edges as a Connected Component (CC). Without considering any modifications to the graph, there is a single CC (the graph). Next, we denote by $\{C_j\}$ the connected components (CCs) that are obtained by pruning from the graph edges that are longer than $\theta$. 

We define a CC to be large if it contains more than $\log (n)$ rewards. Observe that since there are at most $\sqrt{n} \enspace$ CCs (Lemma \ref{lemma:theta}), at least one large CC exists \footnote{To see this, assume that no CC is large than $\log (n)$. Since there are at most $\sqrt{n}$ CCs, this means that in total there are $\log(n)\sqrt{n}<n$ rewards, which stands in contradiction with the fact that there are $n$ rewards. }.  

The proof of step 1 continues as follows. We state two technical Lemmas (Lemma \ref{lemma:budget1} and Lemma \ref{lemma:budget2}) which together lead to Corollary \ref{lemma:budget3}.

\begin{lemma}
\label{lemma:budget1}
Assume that $s_1$ is in a large component $C$. Let $p$ be the path covered by NN starting from $s_1$ until it reaches $s_i$ in a large component.
Let $d$ be the length of $p$ and let $r_1$ be the number of rewards collected by NN in $p$ (including the last reward in $p$ which is back in a large component, but not including $s_1$). Note that $r_1\ge 1$. Then
 $d \le (2^{r_1}-1)\theta$.
\end{lemma}
\begin{proof}
Let $p_i$ be the prefix of $p$ that ends at the $i$th reward on $p$ ($i\le r_1$) and let $d_i$ be the length of $p_i$.
Let $\ell_i$ be the distance from the $i$th reward on $p$ to the $(i+1)$th reward on $p$. 
Since when NN is at the $i$th reward on $p$, the neighbor of $s_1$ in $C$ is at distance at most
$d_i + \theta$ from this reward we have that $\ell_i \le d_i + \theta$. Thus, 
$d_{i+1}\le 2d_i + \theta$ (with the initial condition $d_0 = 0$).
The solution to this recurrence is $d_i = (2^i-1)\theta$.
\end{proof}
\begin{lemma}
\label{lemma:budget2}
For $k<\log (n)$, we have that
after $k$ visits of R-NN in large CCs, for any $s$ in a large CC there exists an unvisited reward at  distance shorter than $(k+1)\theta$ from $s$.
\end{lemma}
\begin{proof}
Let $s$ be a reward in a large component $C$. We have collected at most $k$ rewards from $C$.
Therefore, there exists a reward $s' \in C$ which we have not collected at distance at most $(k+1) \theta$ from $s$. 
\end{proof}
Lemma \ref{lemma:budget1} and Lemma \ref{lemma:budget2} imply the following corollary.
\begin{corollary}
\label{lemma:budget3}
Assume that  $k<\log (n)$, and let
 $p$ be the path of NN from its $k$th reward in a large CC to its $(k+1)$st reward in a large CC.
Let $d$ denote the length of $p$ and 
$r_k$ be the number of rewards on $p$ (excluding the first and including the last). Then 
 $d\le (2^{r_k}-1)(k+1)\theta \le 2^{r_k+1}k\theta $. 
\end{corollary}

The following lemma concludes the analysis of this step.

\begin{lemma}
\label{lemma:sumri2ri}
Let $p$ be the prefix of R-NN of length $x$.
Let $k$ be the number of segments on $p$ of R-NN  that connect rewards in large CCs and contain internally rewards in small CCs.
For $1\le i\le k$, let $r_i$ be the number of rewards R-NN collects in the $i$th segment. Then $\sum_{i=1}^k r_i = \Omega(\log n)$.
(We assume that $p$ splits exactly into $k$ segments, but in fact the last segment may be incomplete, this requires a minor adjustment in the proof.)
\end{lemma}
\begin{proof}
Since $\forall i, r_i\ge 1,$ then if $k \ge \log (n)$ the lemma follows.
So assume that $s < \log n$.
By Corollary \ref{lemma:budget3} we have that 
\begin{equation} \label{eq:hh}
x \le \sum_{i=1}^k 2^{r_i+1}i\theta \le 2^{r_{max}+2}k^2\theta \ .
\end{equation}
where $r_{max} = \mbox{argmax} \enspace \{r_i\}_{i=1}^{k}$.
Since $\theta = x/\sqrt{n}$, Equation (\ref{eq:hh}) implies that 
$\sqrt{n} \le 2^{r_{max}+2}k^2$ and since $k\le \log n$ we get that 
$\sqrt{n} \le 2^{r_{max}+2}\log^2(n)$. Taking logs the lemma follows. 
\end{proof}

Lemma \ref{lemma:sumri2ri} guarantees that once at $s_1 \in C_j,$ R-NN collects  $\Omega(\log (n))$ rewards before traversing a distance of $x$.  Next, notice that the chance that $s_1$ (as defined in Algorithm \ref{alg:nn-bfs}) belongs to one of the large CCs is $p = \frac{n-\sqrt{n}\log (n)}{n}$, 
which is larger than $1/2$ for $n \ge 256$. 

Finally, assume that the value of OPT is greater than a constant fraction of $n$, i.e., $OPT \ge n / 2^\alpha.$ This means that OPT must have collected the first $n / 2^{\alpha+1}$ rewards after traversing a distance of at most $\tilde{x} = \left( \alpha + 1\right)x,$ \footnote{To see this, recall that after traversing a distance of $\tilde{x}$, OPT achieved less than $n / 2^{\alpha+1}$. Since it already traversed $\tilde{x}$ it can only achieve less than $n / 2^{\alpha+1}$ from the remaining rewards, thus a contraction with the assumption that it achieved more than $n / 2^\alpha$.}, and denote this fraction of the rewards by $S_\text{OPT}.$ Further denote by $d_{\text{min}},d_{\text{max}}$ the shortest and longest distances from $s_0$ to $S_\text{OPT}$ respectively. By the triangle inequality, $d_{\text{max}} - d_{\text{min}} \le \tilde{x};$ therefore, with a constant probability of $\frac{1}{2^{\alpha+1}},$ we get that $s_1 \in S_\text{OPT}.$ By taking expectation over the first random pick, it follows that

$$
\frac{\text{R-NN}}{\text{OPT}} \ge \frac{1}{2^{\alpha +1}} \frac{\gamma ^{d_\text{max}} \log (n)}{\gamma^{d_\text{min}} n} =\frac{\log (n)}{4^{\alpha +1}n} =\Omega \left(\frac{\log (n)}{n}\right).
$$

\textbf{Step 2.} We now assume that OPT collects its value from $n'<n$ rewards that it collects in a segment of length $x$ (and from all other rewards OPT collects a negligible value). Recall that the R-NN is either NN with probability $0.5$ or a random pick with probability $0.5$ followed by NN. By picking the single reward closets to the starting point, NN gets at least $1/n'$ of the value of OPT. Notice, that we do not need to assume anything about the length of the tour that OPT takes to collect the $n'$ rewards (since we didn't use it in Step 1). It follows that: $$
\frac{\text{R-NN}}{\text{OPT}} \ge \frac{1}{2} \cdot \frac{1}{n'} + \frac{1}{2} \cdot \frac{n'}{n} \cdot \frac{\log (n')}{n'} = \frac{1}{2n'} + \frac{\log (n')}{n} 
$$
Thus, in the worst case scenario, $n'\log (n') \approx n$,
which implies that 
$n' = \Theta\left(\frac{n}{\log (n)}\right)$. Therefore 
 $\frac{\text{R-NN}}{\text{OPT}} = \Omega \left( \frac{\log (n)}{n} \right),$ which colncludes the proof of Theorem \ref{theo:nn-rand}.

\end{proof}

\section{Impossibility for Stochastic Local Policies}
Theorem \ref{theo:nn-rand} shows that our stochastic modification to NN improves its guarantees by a factor of $\log (n).$ While the improvement over NN may seem small ($\log (n)$), the observation that stochasticity improves the performance guarantees of local policies is essential to our work. These observations motivated us to look for better local policies in the broader class of stochastic local policies. Theorem \ref{theorem:sto_loc_lower} gives a stronger impossibility result for such policies. The MDP that is used in the proof is very similar to the one presented in Figure \ref{fig:worst_case_mdp} (right), but now there are $\sqrt{n}$ rewards in the clique instead of $n/2$.

\begin{theorem}[Impossibility for Stochastic Local Policies]
\label{theorem:sto_loc_lower}
%For each $n > 3$, and 
For each stochastic local policy S-Local, $\exists$ MDP with $n$ rewards and a discount factor $\gamma = 1-1/\sqrt{n},$ s.t.: \enspace  $\frac{\text{S-Local}}{\text{OPT}} \leq \frac{8}{\sqrt{n}}. $
\end{theorem}

\begin{proof}
We consider a family of graphs, ${\cal G}$, each of which consists of a star
with a central vertex and $n$ leaves.
The starting vertex is the central vertex, and there is a reward at each leaf. The length of each edge is $d$, where $d$ is
chosen such that 
 $\gamma^d = \frac{1}{2}$.
 Each graph in ${\cal G}$ corresponds to a subset of 
 $\sqrt{n}$ leaves, which we pairwise connect to form a clique.

Since $\gamma = 1-\frac{1}{\sqrt{n}},$ we have that $\sum_{i=0} ^{\sqrt{n}-1} \gamma ^i \ge \frac{\sqrt{n}}{2},$ and therefore
\begin{align*}
\text{OPT}&=\gamma^d \sum _{i=0}^{\sqrt{n}-1} \gamma ^i + \gamma ^{2d+\sqrt{n}-1} \sum_{i=1} ^{n-\sqrt{n}} \gamma ^ {(2i-1)d} 
 \geq 0.5 \sum _{i=0}^{\sqrt{n}-1}\gamma ^ i \ge 0.25 \sqrt{n}.
\end{align*}

On the other hand, local policy at the central vertex cannot distinguish among the rewards, and therefore for every graph in ${\cal G}$, it picks the first reward from the same distribution. The policy continues to choose rewards from the same distribution until it hits the first reward from the $\sqrt{n}$-size clique. 

To argue formally that every S-Local policy has a small expected reward on a graph from ${\cal G}$, we use Yao's principle \citep{yao1977probabilistic} and consider the expected reward of a D-Local policy on the uniform distribution over ${\cal G}$. 

Let $p_1 = \sqrt{n}/n$ be the probability that D-Local picks its first vertex from the $\sqrt{n}$-size clique. 
Assuming that the first vertex is not in the clique, let $p_2 = \sqrt{n}/(n-1)$ be the probability that the second vertex is from the clique, and let $p_3$, $p_4$, $\ldots$ be defined similarly. When D-local picks a vertex in the clique then its reward (without the cumulative discount) is $O(\sqrt{n})$. However, each time D-Local misses the clique then it collects a single reward but suffers a discount of
$\gamma^{2d} = 1/4$. 
Neglecting the rewards collected until it hits the clique, the total value of D-Local is 
$$
O\left( \left( p_1 + (1-p_1)\gamma^{2d}p_2 + (1-p_1)(1-p_2)\gamma^{4d}p_3 \ldots\right)  \sqrt{n}\right)
$$
Since $p_i \le 2/\sqrt{n}$ for $1\le i\le n/2$ this value is $O(1)$
\end{proof}

We do not have a policy that achieves this lower bound, but we now propose and analyze two stochastic policies (in addition to the R-NN) that substantially improve over the deterministic upper bound. As we will see, these policies satisfy the Occam razor principle, i.e., policies with better guarantees are more complicated and require more computations.

\section{NN with Randomized Depth First Search (RDFS)}
We now describe the NN-RDFS policy (Algorithm \ref{alg:sp-dfs}), our best performing local policy. The policy performs NN with a probability of $0.5$ and a local policy that we call RDFS with a probability of $0.5$. RDFS first collects a random reward and continues by performing DFS\footnote{DFS is an algorithm for traversing a graph. The Algorithm starts at the root node and explores as far as possible along each branch, with the shortest first,
before backtracking.} on edges shorter than $\theta$, where $\theta$ is chosen at random as we specify later. When it runs out of edges shorter than $\theta$, then RDFS continues by performing NN. The performance guarantees for the NN-RDFS method are stated in Theorem \ref{theo:nn-dfs}. 
\begin{algorithm}
\caption{NN with RDFS}
\label{alg:sp-dfs}
\begin{algorithmic}
\STATE {\bfseries Input:} MDP M, with $n$ rewards, and $s_0$ the first reward
\STATE Let $x=\mbox{log}_{\frac{1}{\gamma}}(2)$
\STATE Flip a coin
\IF{outcome = heads} %$\#$Perform RDFS
%\bindent
    \STATE \hskip0.5em Collect a random reward, denote it by $s_1$ 
    \STATE \hskip0.5em Choose at random $ i \sim \text{Uniform} \{ 1,2, ... ,\mbox{log}_2(n) \} $ 
    \STATE \hskip0.5em Fix $n' =\frac{n}{2^i} $ and set $\theta = \frac{x}{\sqrt{n'}} $
    \STATE \hskip0.5em Initiate DFS from $s_1$ on edges shorter than $\theta$ 
%\eindent
\ENDIF
\STATE Follow by executing NN 
\end{algorithmic}
\end{algorithm}
\begin{theorem}[NN-RDFS Performance]
\label{theo:nn-dfs}
For any MDP that satisfies Definition \ref{def:2}, with $n$ rewards,$\forall \gamma:$ 
$
\frac{\text{NN-RDFS}}{\text{OPT}} \ge 
\begin{cases} 
\Omega ( \frac{n^{-\frac{1}{2}}}{ \log^2(n)} ),&\text{if OPT} = \Omega (n). \\
\Omega ( \frac{n^{-\frac{2}{3}}}{\log^2(n)} ),&\text{otherwise}. 
\end{cases} 
$
\end{theorem}

\begin{proof}
\textbf{Step 1.}

Assume that OPT collects a set $S_\text{OPT}$ of $\alpha n$  rewards for some fixed $0\le \alpha \le 1$, in a segment $p$ of length $x = \mbox{log}_{1/\gamma} (2)$ (i.e.\ $x$ is the distance from the first reward to the last reward -- it does not include the distance from the starting point to the first reward). Let $d_{\text{min}},d_{\text{max}}$ the shortest and longest distances from $s_0$ to a reward in $S_\text{OPT}$ respectively. By the triangle inequality, $d_{\text{max}} - d_{\text{min}} \le x.$ 
We further assume  that $\text{OPT} \le O(\gamma^{d_\text{min}}\alpha n)$ (i.e., That is the value that OPT collects from rewards which are not in $S_\text{OPT}$ is negligible).
% \hk{We make two assumptions that we do not justify here..}
We now show that RDFS is $\Omega({\sqrt{n}})$ for $\theta=x/\sqrt{\alpha n}$. 

Lemma \ref{lemma:theta} assures that after pruning all edges longer than $\theta$ (from the graph), $S_\text{OPT}$ breaks into  at most $x / \theta$ CCs. Let $m\le {x}/{\theta}$ be the number of such CCs  $\{ C_j \}_{j=1}^{m}$ containing $S_\text{OPT}$. In addition, it holds that $\sum _{j=1}^{m} |C_j| \ge \alpha n$, and the edges inside any CC $C_j$ are shorter than $\theta$. 

Next, we (lower) bound the total gain of RDFS in its prefix of length $x$ following its initial reward $s_1$. Say $s_1\in C_j$. Then, since all edges in $C_j$ are shorter than $\theta$, it collects in this prefix either all the rewards in $C_j$ or at least $x/(2\theta)$ rewards overall. That is  $\mbox{min} \{ |C_j|, \frac{x}{2\theta}\}$ rewards. To see this, recall that the DFS algorithm traverses each edge at most twice. In addition, as long as it did not collect all the rewards of $C_j$, the length of each edge that the DFS traverses is at most $\theta$. Thus, if the Algorithm did not collect all the rewards in $C_j$ in its prefix of length $x$, then it collected at least $x/2\theta$ rewards.% in this prefix.
%This gives the lower bound of $\min\left(|C_j|, x/(2\theta) \right)$ on the number of rewards that the Algorithm collects in a prefix of length $x$. \\

Notice that the first random step leads RDFS to a vertex in CC $C_j$ with probability $|C_j|/n$. 
We say that a CC $C_j$
is {\em small} if $ |C_j| \le x/ 2\theta$ and we say that it is {\em large} otherwise. Let $S = \{j\mid C_j \; {\rm is \;small} \}$ and let
 $s=|S|$ be  the number of small CCs.  If more than half of the rewards $S_\text{OPT}$ are in small components ($\sum _{j\in S}|C_j|>0.5 \alpha n$), then
\begin{align*}
\mbox{RDFS} & \ge \gamma^{d_\text{max}+x} \sum\nolimits_{j=1}^m \frac{|C_j|}{n} \cdot \mbox{min}\left\{ |C_j|,\frac{x}{2\theta} \right\}  
 \ge  \frac{\gamma^{d_\text{max}}}{2} \sum\nolimits_{j \in S } \frac{|C_j|^2}{n} \ge \frac{s\gamma^{d_\text{max}}}{2n}  \left( \frac{1}{s} \sum\nolimits_{j\in S}  |C_j|^2 \right)  \\ 
& \underset{\mbox{Jensen}}{\ge}  \frac{\gamma^{d_\text{max}}s}{2n} \left( \frac{1}{s}\sum\nolimits_{j\in S} |C_j| \right) ^2  \underset{s \le x/\theta}{\ge}   \gamma^{d_\text{max}} \frac{\theta \alpha^2 n}{8x}. 
\end{align*}
On the other hand, if more than half of rewards are not in $S$ ($\sum _{j\not\in S} |C_j|>0.5 \alpha n$), then 
\begin{align*}
\mbox{RDFS} & \ge  \gamma^{d_\text{max}+x}\sum\nolimits_{j=1}^{x/ \theta} \frac{|C_j|}{n} \cdot \mbox{min}\left\{ |C_j|,\frac{x}{2 \theta} \right\} 
 \ge \frac{1}{2} \gamma^{d_\text{max}}\sum\nolimits_{j\not\in S } \frac{|C_j|}{n} \cdot \frac{x}{2\theta} 
 \ge \gamma^{d_\text{max}}\frac{\alpha x}{8 \theta}.
\end{align*}
By setting $\theta = {x}/{\sqrt{\alpha n}}$ we guarantee that the value of RDFS is at least $\gamma^{d_\text{max}} \sqrt{\alpha^3n} / 8$. Since $d_\text{max}-d_\text{min}\leq x$,
$$
\frac{\text{RDFS}}{\text{OPT}} \ge \frac{ \gamma^{d_\text{max}} \alpha^{3/2} \sqrt{n} / 8}{\gamma^{d_\text{min}} \alpha n} \ge \frac{\sqrt{\alpha} \gamma ^{x}}{8\sqrt{n}} = \frac{\sqrt{\alpha}}{16\sqrt{n}} ,
$$ where the last inequality follows from the triangle inequality. 

\textbf{Step 2.} 
Assume that OPT gets its value from $n'<n$ rewards that it collects in a segment of length $x$ (and from all other rewards OPT collects a negligible value). Recall that the NN-RDFS policy is either NN with probability $0.5$ or RDFS with probability $0.5$. By picking the single reward closest to the starting point, NN gets at least $1/n'$ of the value of OPT. Otherwise, with probability $n'/n$, RDFS starts with one of the $n'$ rewards picked by OPT and then, by the analysis of step 1, if it sets $\theta = x/\sqrt{n'}$, RDFS collects  $1/16\sqrt{n'}$  of the value collected by OPT (we use Step (1) with $\alpha =1$). It follows that \begin{align*}
\frac{\text{NN-RDFS}}{\text{OPT}} \ge \frac{1}{2} \cdot \frac{1}{n'} + \frac{1}{2} \cdot \frac{n'}{n} \cdot \frac{1}{16\sqrt{n'}}  = \frac{1}{2n'} + \frac{\sqrt{n'}}{32n}.
\end{align*}
This lower bound is smallest when $n' \approx n^{2/3}$, in which case NN-RDFS collect $\Omega(n^{-2/3})$ of OPT.  
Notice that since $n'$ is not known to NN-RDFS, it has to be guessed in order to choose $\theta$. This is done by setting $n'$ at random from $n' = n / 2 ^i, i \sim \text{Uniform} \{ 1,2, ..., \mbox{log}_2(n) \} $. This guarantees that with probability $1/\log(n)$ our guess for $n'$ will be off of its true value by a factor of at most $2$. This guess approximates $\theta$ by a factor of at most ${\sqrt{2}}$ of its true value. Finally, these approximations degrade our bounds by a factor of $\log (n).$

\textbf{Step 3.}
Finally, we consider the general case where OPT may collect its value in a segment of length larger than $x$. Notice that the value which OPT collects from rewards that follow the first $\log_2 (n)$ segments of length $x$ in its tour is at  most $1$ (since $\gamma ^{\log_{2} (n)\cdot x} = \frac{1}{n}$). This means that there exists at least one segment of length $x$ in which OPT collects at least  $\frac{\text{OPT}}{\log_{2}(n)}$ of its value. Combining this with the analysis in the previous step, the proof is complete. 
\end{proof}
% we get that
% $$
% \frac{\text{NN-RDFS}}{\text{OPT}} \ge 
% \begin{cases} 
% \Omega \left( \frac{n^{-\frac{1}{2}}}{ \mbox{log}(n)^2} \right), &\text{if} \enspace \text{OPT} = \Omega (n). \\
% \Omega \left( \frac{n^{-\frac{2}{3}}}{\mbox{log}(n)^2} \right), &\text{otherwise}. \\
% \end{cases} 
% $$

\section{NN with a Random Ascent (RA)}
We now describe the NN-RA policy (Algorithm \ref{alg:nn-bfs}). Similar in spirit to NN-RDFS, the policy performs NN with a probability of $0.5$ and local policy, which we call RA with a probability of $0.5$. RA starts at a random node, $s_1$, sorts the rewards in increasing order of their distance from $s_1$, and then collects all other rewards in this order. The Algorithm is simple to implement, as it does not require guessing any parameters (like $\theta$, which RDFS has to guess). However, this comes at the cost of a worse bound.

\begin{algorithm}
\caption{NN with RA}
\label{alg:nn-bfs}
\begin{algorithmic}
\STATE {\bfseries Input:} MDP M, with $n$ rewards, and $s_0$ the first reward
\STATE Flip a coin
\IF{outcome = heads}%, \textbf{then:} %$\#$Perform RDFS
    \STATE \hskip0.5em Collect a random reward, denote it by $s_1$
    \STATE \hskip0.5em Sort the remaining rewards by increasing distances from 
    \STATE \hskip0.5em $s_1$, call this permutation $\pi$ 
    \STATE \hskip0.5em Collect the rewards in $\pi$ in increasing order 
\ELSE
\STATE Execute NN
\ENDIF

%\ENDIF
\end{algorithmic}
\end{algorithm}

The performance guarantees for the NN-RA method are given in Theorem \ref{theo:nn-bfs}. The analysis follows the same steps as the proof of the NN-RDFS Algorithm. We emphasize that here, the pruning parameter $\theta$ is only used for analysis purposes and is not part of the Algorithm. Consequently, we see only one logarithmic factor in the performance bound of Theorem~\ref{theo:nn-bfs} in contrast with two in Theorem~\ref{theo:nn-dfs}.

\begin{theorem}[NN-RA Performance]
\label{theo:nn-bfs}
For any MDP that satisfies Definition \ref{def:2} with $n$ rewards,  $\forall \gamma$:
$
\frac{\text{NN-RA}}{\text{OPT}} \ge 
\begin{cases} 
\Omega ( \frac{n^{-2/3}}{\log(n)} ), &\text{if OPT} = \Omega (n). \\
\Omega ( \frac{n^{-3/4}}{\log(n)} ), &\text{otherwise}. \\
\end{cases} 
$
\end{theorem}

\begin{proof}
\textbf{Step 1.}
Assume that OPT collects   a set $S_\text{OPT}$ of $\alpha n$  rewards for some $0\le \alpha \le 1$, in a segment $p$ of length $x = \mbox{log}_\frac{1}{\gamma} (2)$ (i.e. $x$ is the distance from the first reward to the last reward -- it does not include the distance from the starting point to the first reward). Let $d_{\text{min}},d_{\text{max}}$ the shortest and longest distances from $s_0$ to a reward in $S_\text{OPT}$ respectively. By the triangle inequality, $d_{\text{max}} - d_{\text{min}} \le x.$ 
We further assume  that $OPT \le O(\gamma^{d_\text{min}}\alpha n)$ (i.e., That is the value that OPT collects from rewards which are not in $S_\text{OPT}$ is negligible). 

Let $\theta$ be a threshold that we will fix below, and denote by $\{C_j\}$ the CCs of $S_\text{OPT}$ that are created by deleting edges longer than $\theta$ among vertices of $S_\text{OPT}$. By Lemma  \ref{lemma:theta}, we 
 have at most $x/\theta$ CC.
 
 Assume that RA starts at a vertex of a component $C_j$, such that $|C_j|=k$.
 Since the diameter of $C_j$ is at most $(|C_j|-1)\theta$ then it collects its first $k$ vertices (including $s_1$) within a total distance 
 of $2\sum _{i=2}^{k} (i-1)\theta \le k^2 \theta $. 
 So if $k^2 \theta \le x$ then it collects at least $|C_j|$ rewards before traveling a total distance of $x$, and if 
 $k^2 \theta > x$ it collects at least $\lfloor \sqrt{x/\theta} \rfloor$ rewards. (We shall omit the floor function for brevity in the sequal.) 
 It follows that RA collects  $\Omega \left( \mbox{min} \{ |C_j|, \sqrt{\frac{x}{\theta}}\} \right)$ rewards.
Notice that the first random step leads RDFS to a vertex in CC $C_j$ with probability $\frac{|C_j|}{n}$. If more than half of rewards are in CCs s.t $|C_j| \ge \sqrt{\frac{x}{\theta}},$ then
\begin{align*}
\mbox{RA} \ge & \gamma^{d_\text{max}}\sum_{j=1}^{\frac{x}{\theta}} \frac{|C_j|}{n} \cdot \mbox{min}\left\{ |C_j|,\sqrt{\frac{x}{\theta}} \right\} 
\ge  \gamma^{d_\text{max}}\sum_{j: |C_j| \ge \sqrt{\frac{x}{\theta}} } \frac{|C_j|}{n} \cdot \sqrt{\frac{x}{\theta}} \ge \gamma^{d_\text{max}}\frac{\alpha }{2 }  \sqrt{\frac{x}{\theta}}.
\end{align*}
If more than half of rewards in $S_\text{OPT}$ are in CCs such that $|C_j| \le \sqrt{ \frac{x}{\theta}},$ let $s$ be the number of such CCs and notice that $s \le \frac{x}{\theta}.$ We get that: 
\begin{align*}
\mbox{RA} & = \gamma^{d_\text{max}} \sum_{j=1}^{\frac{x}{\theta}} \frac{|C_j|}{n} \cdot \mbox{min}\left\{ |C_j|,\frac{x}{\theta} \right\}   
 \ge  \gamma^{d_\text{max}} \sum_{j: |C_j| \le \sqrt{ \frac{x}{\theta} } } \frac{|C_j|^2}{n} \ge \frac{s}{n} \gamma^{d_\text{max}} \left( \frac{1}{s} \sum_{j=1}^s  |C_j|^2 \right)  \\ 
& \underset{\mbox{Jensen}}{\ge}  \frac{s}{n} \gamma^{d_\text{max}}  \left( \frac{1}{s}\sum_{j=1}^s |C_j| \right) ^2 \ge  \gamma^{d_\text{max}} \frac{\theta \alpha^2 n}{4x}. 
\end{align*}

By setting $\theta = \frac{x}{n^{2/3}}$ we guarantee that the value of RA is at least $\gamma^{d_\text{max}}\alpha^2 n^{1/3} / 4$. Since $d_\text{max}-d_\text{min}\leq x$,
$$
\frac{\text{RA}}{\text{OPT}} \ge \frac{ \gamma^{d_\text{max}} \alpha^2 n^{1/3} / 4}{\gamma^{d_\text{min}} \alpha n} \ge \frac{\alpha \gamma ^{x}}{4n^{2/3}} = \frac{\alpha}{2n^{2/3}} ,
$$ where the last inequality follows from the triangle inequality.

 \textbf{Step 2.} 
Assume that OPT gets its value from $n'<n$ rewards that it collects in a segment of length $x$ (and from all other rewards OPT collects a negligible value). Recall that the NN-RA policy is either NN with probability $0.5$ or RA with probability $0.5$. By picking the single reward closest to the starting point, NN gets at least $1/n'$ of the value of OPT. Otherwise, with probability $n'/n$, RA starts with one of the $n'$ rewards picked by OPT and then, by the analysis of step 1, if it sets $\theta = \frac{x}{(n')^{2/3}}$, RA collects  $\frac{1}{2(n')^{2/3}} $  of the value collected by OPT (we use Step (1) with $\alpha =1$). Thus, $$
\frac{\text{NN-RA}}{\text{OPT}} \ge \frac{1}{2} \cdot \frac{1}{n'} + \frac{1}{2} \cdot \frac{n'}{n} \cdot \frac{1}{2(n')^{2/3}}  = \frac{1}{2n'} + \frac{(n')^{1/3}}{4n}.
$$

This lower bound is smallest when $n' \approx n^{\frac{3}{4}}$, in which case NN-RA collects $\Omega(n^{-3/4})$ of OPT.

\textbf{Step 3.}
By the same arguments from Step 3 in the analysis of NN-RDFS, it follows that $$
\frac{\text{NN-RA}}{\text{OPT}} \ge 
\begin{cases} 
\Omega \left( \frac{n^{-\frac{2}{3}}}{ \log(n)} \right), &\text{if} \enspace \text{OPT} = \Omega (n). \\
\Omega \left( \frac{n^{-\frac{3}{4}}}{\log(n)} \right), &\text{otherwise}. \\
\end{cases} 
$$
\end{proof}

\section{Simulations}

\subsection{Learning Simulations}

We evaluated our algorithms in an MDP that satisfies Definition \ref{def:2}, and in generalized more challenging settings in which the MDP is stochastic. We also evaluated them throughout the learning process of the options, when the option-policies (and value functions)  are sub-optimal.\footnote{We leave this theoretical analysis  to future work.} Note that with sub-optimal options, the agent may reach the reward in sub-optimal time and even may not reach the reward at all. Furthermore, in stochastic MDPs, while executing an option, the agent may arrive at a state where it prefers to switch to a different option rather than completing the execution of the current option. 
 
Our experiments show that our policies perform well, even when some of our assumptions for our theoretical analysis are relaxed.

\textbf{Setup.} An agent (yellow) is placed in a $50$X$50$ grid-world domain (Figure \ref{fig:res_50}, top left). Its goal is to navigate in the maze and collect the available $45$ rewards (teal) as fast as possible. The agent can move by going up, down, left and right, and without crossing walls (red). In the stochastic scenario, there are also four actions, up, down, left, and right, but once an action, say up, is chosen, there is a $10\%$ chance that a random action (chosen uniformly) will be executed instead of up. We are interested in testing our algorithms in the regime where OPT can collect almost all of the rewards within a constant discount (i.e., $OPT\approx\alpha n$), but, there also exist bad tours that achieve a constant value (i.e., taking the most distance reward in each step); thus, we set $\gamma = 1 - \frac{1}{n}.$

The agent consists of a set of options and a policy over the options. We have an option per reward that learns, by interacting with the environment, a policy that moves the agent from its current position to the reward in the shortest way. We learned the options in parallel using Q-learning. We performed the learning in epochs.

At each epoch, we initialized each option in a random state and performed Q-learning steps either until the option found the reward or until $T=150$ steps have passed. Every phase of $K=2000$ epochs ($300k$ steps), we tested our policies with the available set of options. We performed this evaluation for $L=150$ phases, resulting in a total of $45$M training epochs for each option. At the end of these epochs, the policy for each option was approximately optimal.

\textbf{Options:} Figure \ref{fig:res_50} (top, right), shows the quality of the options in collecting the reward during learning in the stochastic MDP. 
For each of the $L$ phases of $K$ epochs ($300k$ steps), we plot the fraction of the runs in which the option reached the reward (red), and the option \textit{time gap} (blue). The option \textit{time gap} is the time that took the option to reach the goal, minus the deterministic shortest path from the initial state to the reward state. We can see that the options improve as learning proceeds, succeeding to reach the reward in more than $99\%$ of the runs. The \textit{time gap} converges, but not zero, since the shortest stochastic path (due to the $10\%$ random environment and the $\epsilon-$greedy policy with $\epsilon=0.1$) is longer than the shortest deterministic path. 

\textbf{Local policies:} At the bottom of Figure \ref{fig:res_50}, we show the performance of the different policies (using the options available at the end of each of the $L$ phases), measured by the discounted cumulative return. In addition to our four policies, we also evaluated two additional heuristics for comparison. The first, denoted by RAND, is a random policy over the options, which selects an option at random at each step. Rand performs the worst since it goes in and out from clusters.  The second, denoted by OPT (with a slight abuse of notation), is a fast approximation to OPT that uses all the information about the MDP (not local) and computes an approximation to OPT by checking all possible choices of the first two clusters and all possible paths through the rewards that they contain and picks the best. This is a good approximation for OPT since discounting makes the rewards collected after the first two clusters negligible. OPT performs better than our policies because it has the knowledge of the full SMDP. On the other hand, our policies perform competitively, without learning the policy over options at all (zero-shot solution). 

Among the local policies, we can see that NN is not performing well since it is "tempted" to collect nearby rewards in small clusters instead of going to large clusters. The stochastic algorithms, on the other hand, choose the first reward at random; thus, they have a higher chance to reach larger clusters, and consequently, they perform better than NN. R-NN and NN-RDFS perform the best and almost the same because effectively, inside the first two clusters, RDFS is taking a tour, which is similar to the one taken by NN. This happens because $\theta$ is larger than most pairwise distances inside clusters. NN-RA performs worse than the other stochastic algorithms, since sorting the rewards by their distances from the first reward in the cluster introduces an undesired ``zig-zag'' behavior, in which we do not collect rewards which are at approximately the same distance from the first in the right order. 

\begin{figure}[h]
\centering
\begin{subfigure}
    \centering
    \includegraphics[width=0.45\linewidth]{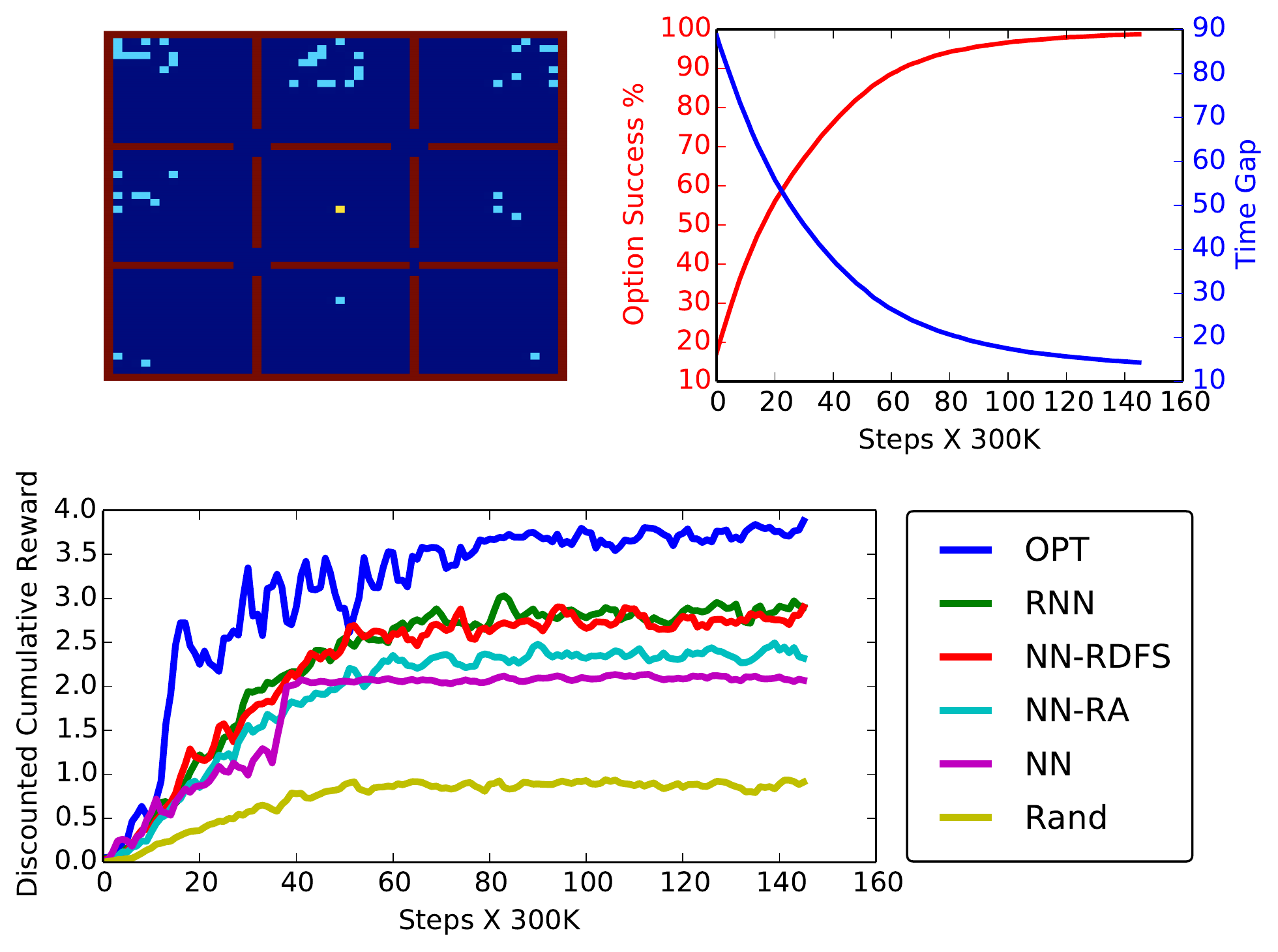}
%     \caption{Stochastic Environment}
\end{subfigure}\hspace{-0.5cm}
\begin{subfigure}
    \centering
    \includegraphics[width=0.45\linewidth]{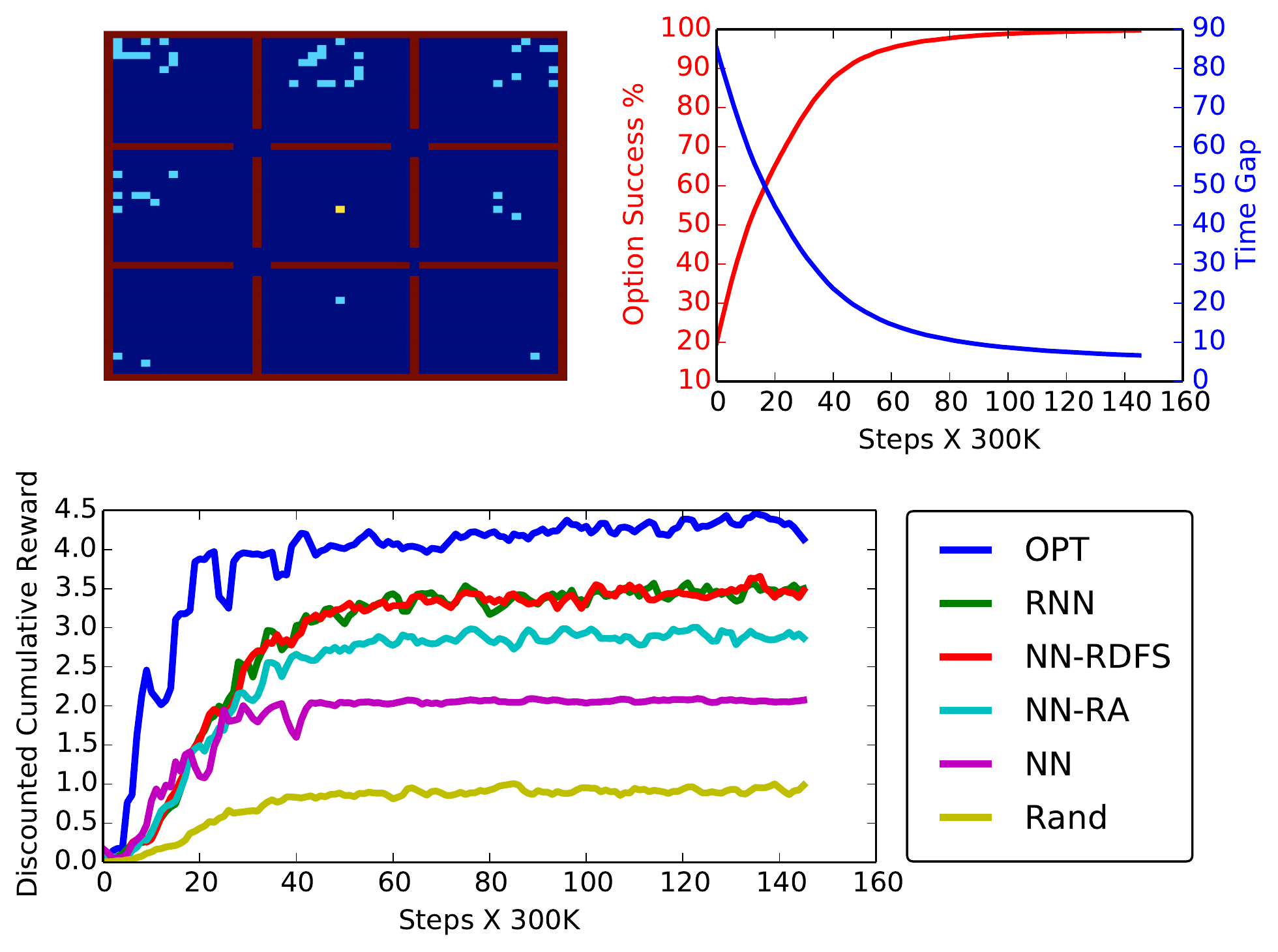}
%     \caption{Deterministic Environment}
\end{subfigure}%
\caption{Stochastic (a) and deterministic (b) environments. In each, on the \textbf{top left,} we can see a grid-world environment, where an agent (yellow) collects rewards (teal) in a maze (walls in red). On the \textbf{top, right,} we can see the options success $\%$ (red) and option time gap (blue) during learning; and at the \textbf{bottom,} we can see the final performance of the different local policies during learning.}
\label{fig:res_50}
\end{figure}

%In the \href{https://drive.google.com/open?id=1N1qHJSIuSZsWsWoD2s4JKNoJaG-UmNTZ}{supplementary}, we provide a visualization of the path taken by each policy, as well as additional experiments with a deterministic MDP. We also provide an additional suite of (planning) experiments, where we evaluate the local policies in the setup of Definition \ref{def:2} on different MDPs. 

\subsection{Planning Simulations}
In this section, we evaluate and compare the performance of deterministic and stochastic local policies by measuring the (cumulative discounted) reward achieved by each algorithm on different MDPs as a function of $n,$ the number of the rewards, with $n\in \{100,200,400,600,800,1000\}.$ For each MDP, the algorithm is provided with the set of optimal options $\left\{o_i \right\}_{i=1}^n$ and their corresponding value functions. We are interested to test our algorithms in the regime where OPT can collect almost all of the rewards within a constant discount (i.e., $OPT\approx\alpha n$), but, there also exist bad tours that achieve a constant value (i.e., taking the most distance reward in each step); thus, we set $\gamma = 1 - \frac{1}{n}.$  We always place the initial state $s_0$ at the origin, i.e., $s_0=(0,0)$. We define $x = \log_{\frac{1}{\gamma}}(2)$, and $\ell=0.01x$ denotes a short distance. 

Next, we describe five MDP types (Figure\ref{fig:vis}, ordered from left to right) that we considered for evaluation. For each of these MDP types, we generate $N_{MDP}=10$ different MDPs, and report the average reward achieved by each Algorithm (Figure \ref{fig:sims}, top), and in the worst-case (the minimal among the $N_{MDP}$ scores) (Figure \ref{fig:sims}, bottom). As some of our algorithms are stochastic, we report average results, i.e., for each MDP, we run each algorithm $N_{alg}=100$ times and report the average score. 

Figure \ref{fig:vis} visualizes these MDPs, for $n=800$ rewards, where each reward is displayed on a 2D grid using gray dots. For each MDP type, we present a single MDP sampled from the appropriate distribution. For the stochastic algorithms, we present the best (Figure \ref{fig:best}) and the worst tours (Figure \ref{fig:vis}), among 20 different runs (for NN we display the same tour since it is deterministic). Finally, for better interpretability, we only display the first $n/k$ rewards of each tour, in which the policy collects most of its value, with $k=8$ $(n/k=100)$ unless mentioned otherwise.\\

\begin{figure}[h]
\centering
\includegraphics[width=0.95\linewidth]{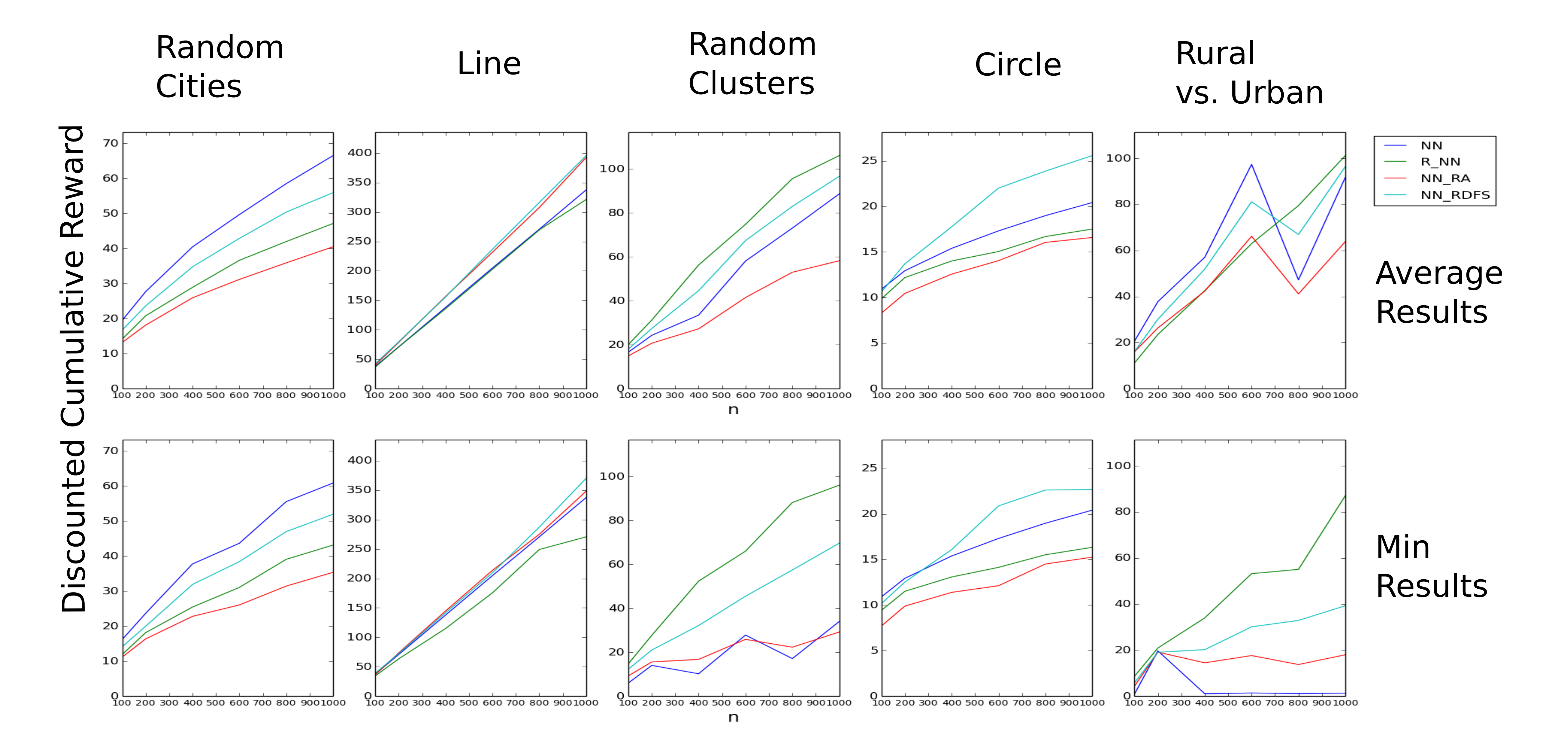}
\caption{Evaluation of deterministic and stochastic local policies over different MDPs. The cumulative discounted reward of each policy is reported for the average and worst case scenarios.}
\label{fig:sims}
\end{figure}
 
\textbf{(1) Random Cities.} For a vanilla TSP with $n$ rewards (nodes) randomly distributed on a $2D$ plane, it is known that the NN Algorithm yields a tour which is  $25\%$ longer than optimal on average \citep{johnson1997traveling}. We used a similar input to compare our algorithms, specifically, we generated an MDP with $n$ rewards $r_i\sim(\text{U}(0,x),\text{U}(0,x)),$ where U is the uniform distribution.

Figure \ref{fig:sims} (left), presents the results for such MDP. we can see that the NN Algorithm performs the best both on the average and in the worst case. This observation suggests that when the rewards are distributed at random, selecting the nearest reward is a reasonable thing to do. In addition, we can see that NN-RDFS performs the best among the stochastic policies (as predicted by our theoretical results). On the other hand, the RA policy performs the worst among stochastic policies. This happens because sorting the rewards by their distances from $s_1$, introduces an undesired ``zig-zag'' behavior while collecting rewards at equal distance from $s_1$ (Figure \ref{fig:vis}).\\

%\textbf{(2) Line.} This MDP demonstrates a scenario where greedy algorithms like NN and R-NN are likely to fail. The rewards are located in three different groups; each contains $n/3$ of the rewards. In group 1, the rewards are located in a cluster left to the origin $r_i\sim(\text{U}[-\theta/3-\ell,-\theta/3+\ell],\text{N}(0,\ell)),$ while in group 2 they are located in a cluster right to the origin $r_i\sim(\text{U}[\theta/3-3\ell,\theta/3-2\ell],\text{N}(0,\ell))$ but a bit closer than group 1 ($\theta=x/\sqrt{n}$). Group 3 is also located to the right, but the rewards are placed in increasing distances, such that the $i$-th reward is located at $(\theta/3)2^i$. \\
%Inspecting the results, we can see that NN and R-NN indeed perform the worst. To understand this, consider the tour that each Algorithm takes. NN goes to group 2, then 3 then 1 (and loses a lot from going to 3). The stochastic tours depend on the choice of $s_1$. If it belongs to group 1, they collect group1 then 2, then 3, from left to right, and perform relatively the same. If it belongs to group 3, they will first collect the rewards to the left of $s_1$ in ascending order and then come back to collect the remaining rewards to the right, performing relatively the same. However, if $s_1$ is in group 2, then NN-RDFS, NN-RA will visit group 1 before going to 3, while R-NN is tempted to go to group 3 before going to 1 (and loses a lot from doing it).\\
\textbf{(2) Line.} This MDP demonstrates a scenario where greedy algorithms like NN and R-NN are likely to fail. The rewards are located in three different groups; each contains $n/3$ of the rewards. In group 1, the rewards are located in a cluster left to the origin $r_i\sim(\text{U}[-\theta/3-\ell,-\theta/3+\ell],\text{N}(0,\ell)),$, while in group 2 they are located in a cluster right to the origin but a bit closer than group 1 $r_i\sim(\text{U}[\theta/3-3\ell,\theta/3-2\ell],\text{N}(0,\ell))$. Group 3 is also located to the right, but the rewards are placed in increasing distances, such that the $i$-th reward is located at $(\theta/3)2^i$. 

For visualization purposes, we added a small variance in the locations of the rewards at groups 1 and 2 and rescaled the axes. The two vertical lines of rewards represent these two groups, while we cropped the graph such that only the first few rewards in group 3 are observed. Finally, we chose $k=2$, such the first half of the tour is displayed, and we can see the first two groups visited on each tour. %{\bf [[YM: This is the only place where $k\neq 8$?]]}

Inspecting the results, we can see that NN and R-NN indeed perform the worst. To understand this, consider the tour that each Algorithm takes. NN goes to group 2, then 3 then 1 (and loses a lot from going to 3). The stochastic tours depend on the choice of $s_1$. If it belongs to group 1, they collect group1 then 2, then 3, from left to right, and perform relatively the same. If it belongs to group 3, they will first collect the rewards to the left of $s_1$ in ascending order and then come back to collect the remaining rewards to the right, performing relatively the same. However, if $s_1$ is in group 2, then NN-RDFS, NN-RA, will visit group 1 before going to 3, while R-NN is tempted to go to group 3 before going to 1 (and loses a lot from doing it).

\textbf{(3) Random Clusters.} This MDP demonstrates the advantage of stochastic policies. We first randomly place $k=10$ cluster centers $c^j$, $j=1,\ldots,k$ on a circle of radius $x$. Then to draw a reward $r_i$ we first draw a cluster center $c^j$ uniformly and then draw $r_i$ such that $ r_i \sim (\text{U}[c^j_x-10\ell,c^j_x+10\ell],\text{U}[c^j_y-10\ell,c^j_y+10 \ell])$.

This scenario is motivated by maze navigation problems, where collectible rewards are located at rooms (clusters) while in between rooms, there are fewer rewards to collect (similar to Figure \ref{fig:res_50}). Inspecting the results, we can see that NN-RDFS and R-NN perform the best, in particular, in the worst-case scenario. The reason for this is that NN picks the nearest reward, and most of its value comes from rewards collected at this cluster. On the other hand, the stochastic algorithms visit larger clusters first with higher probability and achieve higher value by doing so.\\

\begin{figure}[h]
\centering
\includegraphics[width=\textwidth]{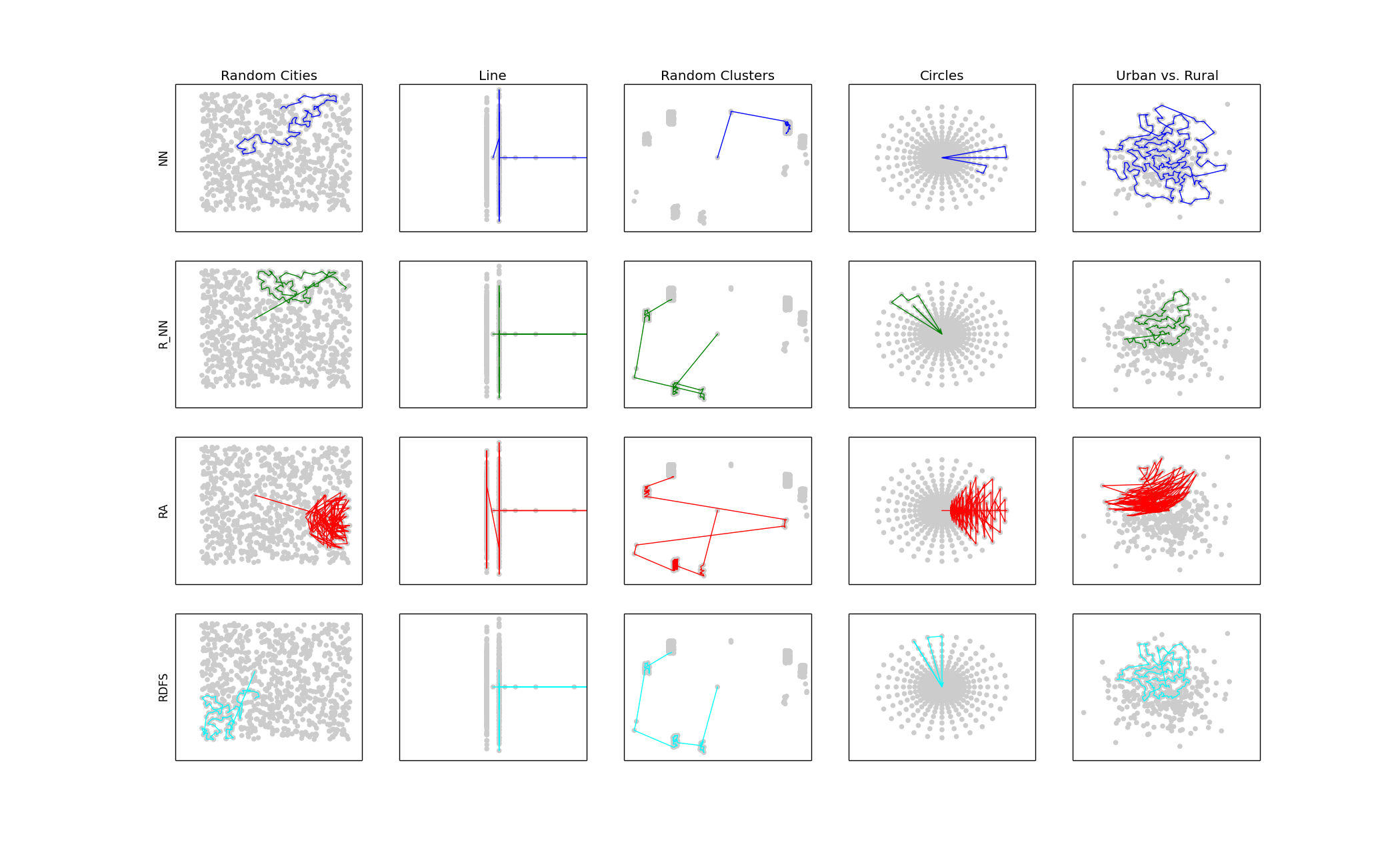}
\caption{Visualization of the worst tours, taken by deterministic and stochastic local policies for $n=800$ rewards.}
\label{fig:vis}
 \end{figure}
  %\textbf{(4) Circle.} In this MDP, there are $\sqrt{n}$ circles, all centered at the origin, and the radii of the $i$th circle is $\rho_i = \frac{x}{\sqrt{n}}\cdot (1+\frac{1}{\sqrt[4]{n}})^i$  On each circle we place $\sqrt{n}$ rewards are placed at equal distances.\\  Here, NN-RDFS performs the best among all policies since it collects rewards closer to $s_1$ first. The greedy algorithms, on the other hand, are ``tempted'' to collect rewards that take them to the outer circles which results in lower values. \\
  
 \textbf{(4) Circle.} In this MDP, there are $\sqrt{n}$ circles, all centered at the origin, and the radii of the $i$th circle is $\rho_i = \frac{x}{\sqrt{n}}\cdot (1+\frac{1}{\sqrt[4]{n}})^i$. On each circle we place $\sqrt{n}$ rewards place at equal distances. This implies that the distance between adjacent rewards on the same circle is longer than the distance between adjacent rewards on two consecutive circles. 
 
Examining the tours, we can see that indeed NN and R-NN are taking tours that lead them to the outer circles. On the other hand, RDFS and RA are staying closer to the origin. Such local behavior is beneficial for RDFS, which achieves the best performance in this scenario. However, while RA performs well in the best case, its performance is much worse than the other algorithms in the worst case. Hence, its average performance is the worst in this scenario. \\

\textbf{(5) Rural vs. Urban.} Here, the rewards are sampled from a mixture of two normal distributions. Half of the rewards are located in a ``city'', i.e., their position is a Gaussian random variable with a small standard deviation s.t. $r_i\sim(\text{N}(x,\ell),\text{N}(0,\ell))$; the other half is located in a ``village'', i.e., their position is a Gaussian random variable with a larger standard deviation s.t. $r_i\sim(\text{N}(-x,10x),\text{N}(0,10x)).$ To improve the visualization here, we chose $k=2$, such the first half of the tour is displayed. Since half of the rewards belong to the city, choosing $k=2$ ensures that any tour that is reaching the city only the first segment of the tour (until the tour reaches the city) will be displayed.

In this MDP, we can see that in the worst-case scenario, the stochastic policies perform much better than NN. This happens because NN is mistakenly choosing rewards that take it to remote places in the rural area, while the stochastic algorithms remain near the city with high probability and collect its rewards.

\begin{figure}[h]
\centering
\includegraphics[width=\textwidth]{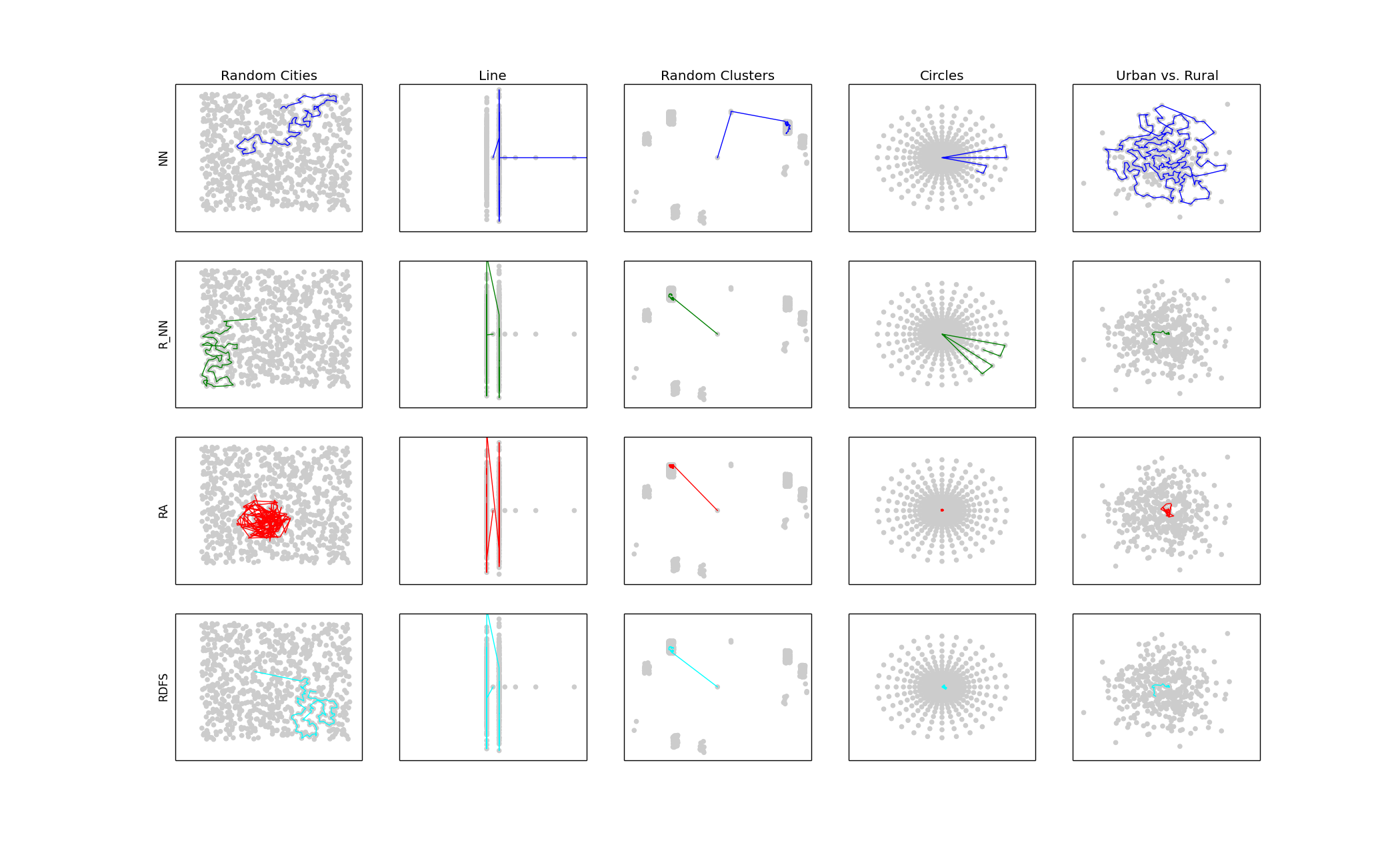}
\caption{Visualization of the best tours, taken by deterministic and stochastic local policies over different RD-TSPs.}
\label{fig:best}
\end{figure}

\section{Related work}
\label{sec:related}

% multi-task decomposition
{\bf Pre-defined rules for option selection} are used in several studies. Karlsson et. al. \citep{karlsson1997learning} suggested a policy that chooses greedily with respect to the sum of the local Q-values $\mbox{a}^*={\mbox{argmax}_a} \sum _i Q_i(s,a)$. Humphrys et. al.  \citep{humphrys1996action} suggested to choose the option with the highest local Q-value $\mbox{a}^*={\mbox{argmax}_{a,i}} Q_i(s,a)$ (NN). 

Barreto et al. \citep{barreto2016successor} considered a general transfer learning problem in RL, where the dynamics are shared along with a set of MDPs, and the reward in the $i$-th MDP is linear in some reward features $R_i(s) = w_i ^T \phi (s)$. They suggested using NN (pick the option of maximum value) as the predefined rule for option selection (but referred to it as General Policy Improvement (GPI)), and provided performance guarantees for using GPI in the form of additive (based on regret) error bounds. In contrast, we prove multiplicative performance guarantees for NN and for our stochastic policies. We also proved, for the first time, impossibility results for such local selection rules. Since our Definition \ref{def:2} is a special case of the framework of \citep{barreto2016successor}, our impossibility results apply to their framework as well.

A different approach to tackle these challenges is \textbf{Multi-task learning,} in which we optimize the options in parallel with the policy over options \citep{russell2003q,sprague2003multiple,van2017hybrid}. One method that achieves that goal is the local SARSA algorithm \citep{russell2003q,sprague2003multiple}. Similar to \citep{karlsson1997learning}, a Q function is learned locally for each option (concerning a local reward). However, here the local Q functions are learnt on-policy (using SARSA) with respect to the policy over options $\pi(s)={\mbox{argmax}_a} \sum _i Q_i(s,a),$ instead of being learned off-policy with Q learning. Russel et al. \citep{russell2003q} showed that if the policy over options is being updated in parallel with the local SARSA updates, then the local SARSA algorithm is promised to converge to the optimal value function.

\section{Discussion on discounting}
Throughout the paper, we focused on values of \textbf{$\gamma$} that allows OPT to collect almost all of the rewards within a constant discount (e.g., $\text{OPT}\approx\alpha n$), but, there also exist bad tours that achieve a constant value (i.e., taking the most distance reward in each step). We now briefly discuss what happens for other values of $\gamma$. 

In the limit that $\gamma \rightarrow 0,$ collecting the first reward fastest is most important, and the greedy solution (NN) is near-optimal (optimal when $\gamma=0$ ). 

On the other extinct, when $\gamma \rightarrow 1$, one may think that the RD-TSP is mapped to TSP, however, this is not the case. When $\gamma = 1$, the value function does not depend on the distances between the rewards, and therefore, any policy that visits all the rewards is optimal. To see what happens in the limit, we examine the first order Taylor approximation of $\gamma ^ t$ around $\gamma = 1$ is given by  $\gamma ^ t \approx (1-(1-\gamma) \cdot t)$, thus, $\gamma ^ {\sum t_i} \approx 1- (1-\gamma) \cdot \sum t_i .$ Let's examine what happens to the value function in this case: $V({1,..,n}) = \sum _{i=1} ^n \gamma ^{\sum _{j=i+1} ^n d_{i,j}} \approx \sum _{i=1} ^n 1- (1-\gamma) \cdot \sum _{j=i+1} ^n d_{i,j} $ Thus, for $\gamma \rightarrow 1$ maximizing $V$ is equivalent to minimizing $\sum _{i=1} ^n \sum _{j=i+1} ^n d_{i,j}.$ This formulation is more similar to the TSP, but now each time we visit a city (collect  a reward), the distance to that city is added to all of the distances of the left cities. This means that given a specific tour, the order in which one visits the cities is important, unlike the TSP.

\section{Conclusions}
We establish theoretical guarantees for policies to collect rewards, based on reward decomposition in deterministic MDPs. Using reward decomposition, one can learn many (option) policies in parallel and combine them into a composite solution efficiently. In particular, we focus on approximate solutions that are local, and therefore, easy to implement and do not require substantial computational resources. Local deterministic policies, like NN, are being used for hierarchical RL. Our study provides theoretical guarantees on the reward collected by these policies, as well as impossibility results. Our theoretical results show that these policies outperform NN in the worst case.
 
We tested our policies in a practical maze navigation setup. Our experiments show that our randomized local policies work well compared to the optimal policy and better than the NN policy. Furthermore, we demonstrated that this also holds throughout the options' learning process (when their policies are suboptimal), and even when the actions are stochastic. We expect to see similar results if each option will be learned with function approximation techniques like DNNs \citep{tessler2017deep,bacon2017option}.

\clearpage

\bibliography{main}
%\end{document}
\clearpage

\section{Exact solutions for the RD-TSP}
\label{sec:exact}
We now present a variation of the Held-Karp Algorithm for the RD-TSP. Note that similar to the TSP, $C(\{S,k\})$ denotes the length of the tour visiting all the cities in $S$, with $k$ being the last one (for TSP, this is the length of the shortest tour). However, our formulation required the definition of an additional recursive quantity, $V(\{S,k\})$, that accounts for the value function (the discounted sum of rewards) of the shortest path. Using this notation, we observe that Held-Karp is identical to doing tabular Q-learning on SMDP $M_s$. Since Held-Karp is known to have exponential complexity, it follows that solving $M_s$ using SMDP algorithms is also of exponential complexity. 

\begin{algorithm}[ht]
\caption{The Held-Karp for the TSP (blue) and RD-TSP (black)}
\label{alg:held-karp}
\begin{algorithmic}
\STATE {\bfseries Input:} Graph $G$, with $n$ nodes

\FOR{$k$ := $2$ to $n$}
\STATE    $\color{blue}{C(\{k\}, k) := d_{1,k}}$
\STATE    $C(\{k\}, k)$ := $d_{1,k}$
\STATE    $V(\{k\}, k)$ := $\gamma^{d_{1,k}} $
\ENDFOR
\FOR{$s$ := $2$ to $n-1$}
    \FOR{all $S \subseteq \{2, . . . , n\}, |S| = s$}
         \FOR{all $k \in S$}
             \STATE $\color{blue}{C(S, k) =\mbox{min}_{m \ne k ,m \in s } [C(S         \setminus \{k\}, m) + d_{m,k}]}$
            \STATE $Q(S,k,a) = [V(S \setminus \{k\}, a) + \gamma ^{C(S \setminus \{k\},a)} \cdot \gamma ^{d_{a,k}} ]$
            \STATE $a^*$ = $\mbox{arg max}_{a \ne k ,a \in s } \enspace Q(S,k,a)$
            \STATE $C(S, k) = C(S \setminus \{k\},a^*) + d_{a^*,k}$
            \STATE $V(S, k) = Q(S,k,a^*)$
        \ENDFOR
    \ENDFOR
\ENDFOR
\STATE $\color{blue}{opt := \mbox{min}_{k\ne 1} [C(\{2,  . . . , n\}, k) + d_{k,1}]}$\\
\STATE $opt := \mbox{max}_{k\ne 1} [V(\{2, \ldots , n\},k) + \gamma ^{C(\{2,  \ldots , n\}, k)+d_{k,1}} $\\
return ($opt$)\\
\end{algorithmic}
\end{algorithm}

\subsection{Exact solutions for simple geometries }
We now provide exact, polynomial-time solutions based on dynamic programming for simple geometries, like a  line and a star. We note that such solutions (exact and polynomial) cannot be derived for general geometries.
%as we will see in the proofs of the following Sections, even for a 2d grid. 

\textbf{Dynamic programming on a line (1D):} Given an RD-TSP instance, such that all the rewards are located on a single line (denoted by the integers $1,\ldots,n$ from left to right), it is easy to see that an optimal policy collects, at any time, either the nearest reward to the right or the left of its current location. Thus, at any time, the set of rewards that it has already collected lie in a continuous interval between the first uncollected reward to the left of the origin denoted $\ell$,  and the first uncollected reward to the right of the origin denoted $r$.  The action to take next is either to collect $\ell$ or to collect $r$.

It follows that the state of the optimal agent is uniquely defined by a triple $(\ell,r,c)$, where $c$ is the current location of the agent. Observe that $c \in \{\ell+1,r-1\} $  and therefore there are $O(n^2)$ possible states in which the optimal policy could be. 

Since we were able to classify a state space of polynomial-size, which contains all states of the optimal policy, then we can describe a 
dynamic programming scheme (Algorithm \ref{alg:line}) that finds the optimal policy. The Algorithm
computes a table $V$, where
$V(\ell,r,\rightarrow)$ is
the maximum value we can get by collecting all rewards $1,\ldots \ell$ and $r,\ldots, n$
starting from $r-1$, and  
$V(\ell,r,\leftarrow)$ is defined analogously starting from $\ell+1$. The Algorithm
 first initializes the entries of  $V$ 
 where either $\ell = 0$ or $r=n+1$. These entries correspond to the cases where all the rewards to the left (right) of the agent have been collected.
(In these cases, the agent continues to collect all remaining rewards one by one in their order.) It then iterates over $t$, a counter over the number of rewards that are left to collect. For each value of $t$, we define $S$ as all the combinations of partitioning these $t$ rewards to the right and the left of the agent.
We fill $V$ by increasing the value of $t$. 
To fill an entry $V(\ell,r,\leftarrow)$  such that $\ell + (n+1 -r) = t$ we take the largest among 1) the value to collect $\ell$ and then
the rewards $1,\ldots \ell-1$ and $r,\ldots, n$, appropriately discounted and 2) the value to collect $r$ and then 
$1,\ldots \ell$ and $r+1,\ldots, n$. 
We fill $V(\ell,r,\rightarrow)$ analogously.

The optimal value for starting position $j$ is $1+V(j-1,j+1,\rightarrow)$.
Note that the Algorithm computes the value function; to get the policy, one has merely to track the argmax at each maximization step.

\begin{algorithm}[ht]
\caption{Exact solution on a line}
\label{alg:line}
\begin{algorithmic}
\STATE {\bfseries Input:} Graph G, with n nodes
\STATE Init $V(\cdot,\cdot,\cdot)=0$
\FOR{$t=1,n$}
    \STATE $V(\ell=t,n+1,\rightarrow) := \gamma^{d_{t,n}} \cdot \left(1+\sum _{j=t} ^{2}         \gamma^{\sum _{i=0}^{j} d_{i,i-1}}\right)  $ 
    \STATE $V(0,r=t,\leftarrow) := \gamma^{d_{1,t}} \cdot \left(1+\sum _{j=t} ^{n-1}             \gamma^{\sum _{i=0}^{j} d_{i,i+1}}\right)  $ 
\ENDFOR
\FOR{$t = 2, .. n-1  $}
    \STATE $S = \{ (i,n + 1 - j) | i+j = t \}$
    \FOR{$(\ell,r) \in S$}
        \IF{$V(\ell,r,\leftarrow) = 0$}          
        \STATE \begin{equation*}
        V(\ell,r,\leftarrow) = \mbox{max} \begin{cases}
                \gamma ^{d_{\ell,\ell+1}}  \left[ 1 + V(\ell-1,r,\leftarrow) \right] \\
                \gamma^{d_{\ell+1,r}} \left[ 1 + V(\ell,r+1,\rightarrow) \right] 
             \end{cases}
        \end{equation*}
        \ENDIF                   
        \IF{$V(\ell,r,\rightarrow) = 0$} 
        \STATE \begin{equation*} 
            V(\ell,r,\rightarrow) = \mbox{max} \begin{cases}
                \gamma ^{d_{\ell,r-1}}  \left[ 1 + V(\ell-1,r,\leftarrow) \right] \\
                \gamma^{d_{r-1,r}} \left[ 1 + V(\ell,r+1,\rightarrow) \right] 
             \end{cases}
        \end{equation*}
        \ENDIF                     
    \ENDFOR
\ENDFOR
\end{algorithmic}
\caption{Optimal solution for the RD-TSP on a line.
The rewards are denoted by $1,\ldots, n$ from left to right. We denote by $d_{i,j}$  the distance between reward $i$ and  reward $j$. 
We denote by $V(\ell,r,\rightarrow)$
the maximum value we can get by collecting all rewards $1,\ldots, \ell$ and $r,\ldots,n$ starting from reward $r-1$. 
 Similarly, we  denote by
$V(\ell,r,\leftarrow)$ maximum value we can get by collecting all rewards $1,\ldots, \ell$ and $r,\ldots,n$
starting from $\ell +1$.
If the leftmost (rightmost) reward was collected we define $\ell=0$ ($r=n+1$).}
\end{algorithm}
\clearpage
\textbf{Dynamic programming on a $d$-star:} We consider an RD-TSP instance, such that all the rewards are located on a d-star, i.e., all the rewards are connected to a central connection point via one of $d$ lines, and there are $n_i$ rewards along the i$th$ line. We denote the rewards on the $i$th line by $m^i_j \in \{1,..,n_i\},$ ordered from the origin to the end of the line, and focus on the case where the agent starts at the origin.\footnote{The the more general case is solved by applying Algorithm \ref{alg:line} until the origin is reached followed by Algorithm \ref{alg:dstar}} It is easy to see that an optimal policy collects, at any time, the uncollected reward that is nearest to the origin along one of the $d$ lines. Thus, at any time, the set of rewards that it has already collected lie in $d$ continuous intervals between the origin and the first uncollected reward along each line, denoted by $\bar{\ell}=\{\ell_i\}_{i=1}^d$. The action to take next is to collect one of these nearest uncollected rewards. It follows that the state of the optimal agent is uniquely defined by a tuple $(\bar{\ell},c)$, where $c$ is the current location of the agent. Observe that $c \in \{m^i_{\ell_i-1} \}_{i=1}^d$  and therefore there are $O(dn^{d})$ possible states in which the optimal policy could be. 

\begin{algorithm}[ht]
\caption{Exact solution on a d-star}
\label{alg:dstar}
\begin{algorithmic}
\STATE {\bfseries Input:} Graph G, with n rewards.
\STATE Init $V(\cdot,\cdot)=0$
\FOR{$i \in \{1,..,d\}$}
\FOR{$\ell_i \in \{1,..,n_i\}$}
\FOR{$c \in \{m^1_{n_1},..,m^i_{\ell_i-1},m^d_{n_d}\}$}
        \STATE $V\left(n_1+1,..,l_i,n_d+1,c\right) =  \gamma^{d_{c,m^i_{\ell_i}}}\cdot\left(1+\sum _{j=\ell_i}                                 \gamma^{\sum _{k=0}^{j}d_{k,k+1}}\right)$ 
\ENDFOR
\ENDFOR
\ENDFOR
\FOR{$t = 2, .. n-1  $}
    \STATE $S = \{ \bar{\ell} | \ell_i \in \{1,..,n_i\} \sum \ell_i = n-t \}$
    \FOR{$\bar{\ell} \in S, c \in \{m^i_{\ell_i-1} \}_{i=1}^d$}
        \IF{$V(\bar{\ell},c) = 0$}  
            \STATE $A = \{ m^i_j | j = \ell_i , j\le n_i\} $ 
            \FOR{$a \in A $}
                    \STATE \begin{equation*}
                    V(\bar{\ell},c)= \mbox{max} \begin{cases} 
                    V(\bar{\ell},c)\\
                    \gamma ^{d_{c,a}}  \left[ 1 + V(\bar{\ell}+e_a,a) \right]
                    \end{cases}\end{equation*}
            \ENDFOR
        \ENDIF                   
    \ENDFOR
\ENDFOR
\end{algorithmic}
\caption{Optimal solution for the RD-TSP on a d-star. We denote by $n_i$ the amount of rewards there is to collect on the $i$th line, and denote by $m^i_j \in \{1,..,n_i\}$ the rewards along this line, from the center of the star to the end of that line. We denote by $d_{m^t_i,m^k_j}$  the distance between reward $i$ on line $t$ and reward $j$ on line $k$. The first uncollected reward along each line is denoted by $\ell_i$,
and the maximum value we can get by collecting all the remaining rewards $m^1_{\ell_1} \ldots m^1_{n_1},\ldots,m^d_{\ell_d},\ldots,m^d{n_d}$ starting from reward $c$ is defined by $V(\bar{\ell}=\{\ell_i\}_{i=1}^d,c)$.  If all the rewards were collected on line $i$ we define $\ell_i=n_i+1$.}
\end{algorithm}

Since we were able to classify a state space of  polynomial size which contains all states of the optimal policy then we can describe a dynamic programming scheme (Algorithm \ref{alg:dstar}) that finds the optimal policy. The algorithm computes a table $V$, where $V(\bar{\ell},c)$ is the maximum value we can get by collecting all rewards $\{m^i_{\ell_i} ,\ldots, m^i_{n_i}\}_{i=1}^d$ starting from $c$. The algorithm first initializes the entries of  $V$  where all $\ell_i = n_i+1$ except for exactly one entry. These entries correspond to the cases where all the rewards have been collected, except in one line segment (in these cases the agent continues to collect all remaining rewards one by one in their order.) It then iterates over $t$, a counter over the number of rewards that  are left to collect. For each value of $t$, we define $S$ as all the combinations of partitioning these $t$ rewards among $d$ lines. We fill $V$ by increasing value of $t$.  To fill an entry $V(\bar{\ell},c)$  such that $\sum l_i = n-t$ we take the largest among the values for collecting $\ell_i$ and then the rewards $m^1_{\ell_1} \ldots m^1_{n_1},\ldots, m^i_{\ell_i+1}, \ldots m^i_{n_i},\ldots,m^d_{\ell_d},\ldots,m^d_{n_d} $ appropriately discounted. 

Note that the Algorithm computes the value function; to get the policy, one has merely to track the argmax at each maximization step.

\end{document}